\documentclass[11pt]{article}

\usepackage[final]{acl}

\usepackage{times}
\usepackage{latexsym}

\usepackage[T1]{fontenc}

\usepackage[utf8]{inputenc}

\usepackage{microtype}

\usepackage{inconsolata}

\usepackage{graphicx}

\title{MOA: Multi-Objective Alignment for Role-Playing Agents
}

\author{Chonghua Liao$^1$, Ke Wang$^{2}$  \\
  \textbf{Yuchuan Wu}$^{2}$, 
  \textbf{Ruoran Li}$^1$,
  \textbf{Fei Huang}$^2$, 
  \textbf{Yongbin Li}$^{2\dagger}$\\
  $^1$ Tsinghua University, 
$^2$ Tongyi Lab \\
\texttt{lch22@mails.tsinghua.edu.cn} \quad \texttt{wk258730@alibaba-inc.com}
}

\usepackage{caption}
\usepackage{float}
\usepackage{times}
\usepackage{latexsym}
\usepackage{multirow}
\usepackage{booktabs}
\usepackage{dsfont}
\usepackage{enumitem}

\usepackage{graphicx}
\usepackage{float}
\usepackage{subfigure}
\usepackage{subcaption}
\usepackage{tcolorbox}
\usepackage{comment}
\usepackage{pifont}
\usepackage{titletoc}
\usepackage{amsfonts}
\usepackage{booktabs}
\usepackage{arydshln}
\usepackage{colortbl}
\usepackage{algorithm}
\usepackage{algorithmicx}
\usepackage[dvipsnames]{xcolor}
\usepackage{wrapfig}

\usepackage[noend]{algpseudocode}

\usepackage{amssymb,amsthm}
\usepackage{bm}
\usepackage{physics}
\usepackage{hyperref}
\usepackage{mathtools}

\usepackage{enumitem}
\usepackage{graphicx}
\usepackage{soul}
\usepackage{colortbl,array,xcolor}
\makeatletter
\definecolor{darkgreen}{RGB}{0,100,0}
\newcommand{\inlinecomment}[1]{\Statex \hskip\ALG@thistlm \textcolor{darkgreen}{\# #1}}
\makeatother
\usepackage{amsmath}
\usepackage{fontawesome}
\usepackage{multirow}

\newcommand{\short}{MOA}

\newtheorem{assumption}{Assumption}
\newtheorem{theorem}{Theorem}
\newtheorem{corollary}{Corollary}

\begin{document}
\maketitle
\begin{abstract}
Role-playing agents (RPAs) require balancing multiple objectives, such as instruction following, persona consistency, and stylistic fidelity, which are not always perfectly aligned across different dimensions. While prior work has primarily relied on supervised fine-tuning or reinforcement learning with scalarized rewards, these approaches do not explicitly address the coordination of multiple reward dimensions during optimization. We present \textbf{MOA} (\textbf{M}ulti-\textbf{O}bjective \textbf{A}lignment), a reinforcement-learning framework that enables multi-dimensional, fine-grained rubric optimization for general RPAs. MOA introduces a novel multi-objective optimization strategy that trains simultaneously on multiple fine-grained rubrics to boost optimization performance. Additionally, to improve both output diversity and generation quality, we employ thought-augmented rollouts with off-policy guidance. Experiments on PersonaGym and RoleMRC show that MOA consistently improves multi-dimensional role-playing performance over supervised and standard RL baselines. Under identical evaluation protocols, an 8B model trained with MOA reaches performance competitive with strong closed-source models across multiple evaluation dimensions. These results suggest that MOA provides a practical framework for training more capable general-purpose role-playing agents.

\end{abstract}

\section{Introduction}\label{sec:intro}
Role-playing agents (RPAs) have become an increasingly active research area, drawing attention from both academia and industry. Advances in large language models have made RPAs viable for a range of interactive applications, including customer-service systems, content generation, interactive entertainment, and conversational non-player characters (NPCs) in digital games~\citep{shao2023character,wang2021naturalconv,liu2024speak,xu2024can}.

Prior work on role-playing agents has explored several complementary directions. A substantial line of research focuses on \textbf{evaluation}, proposing benchmarks and metrics to assess role-playing ability from different perspectives~\citep{lu2024large,yang2024simschat,samuel2024personagym,lu2025rolemrc}. In parallel, many studies emphasize \textbf{data-centric approaches}, where supervised fine-tuning (SFT) on synthetic or curated dialogues is used to improve general role-playing behavior~\citep{wang2025coser,tang2025thinking,wang2025opencharacter}.

Currently, SFT remains the dominant paradigm for training RPAs. However, reliance on SFT alone exhibits notable limitations. First, SFT tends to overfit surface-level patterns in the training data, resulting in limited generalization~\citep{wang2025raiden,tang2025thinking}. Second, SFT often constrains output diversity, which has been shown to hinder subsequent optimization and exploration~\citep{cui2025entropy,wang2025beyond}. 

Beyond SFT, several works have explored reinforcement learning (RL) for improving role-playing agents by transferring techniques originally developed for reasoning tasks. For instance, RAIDEN-R1~\citep{wang2025raiden} formulates role-playing optimization by treating keyword matching as a verifiable reward signal. Such approaches do not fully capture two fundamental properties of role-playing. First, role-playing performance is inherently \textbf{multi-dimensional}, requiring fine-grained rewards to reflect different aspects of the response. Second, reward dimensions in role-playing are often weakly correlated or even conflicting, such that improving one dimension may degrade others. A common example arises between role knowledge and persona style: responses that are long and structured (e.g., bullet-pointed) tend to score highly on knowledge-related criteria, yet often deviate from the intended persona style.

Under such conditions, scalarizing multiple rewards into a single objective can obscure dimension-specific learning signals. This issue becomes particularly apparent when applying standard weighted Group Relative Policy Optimization~(GRPO)~\citep{guo2025deepseek}. As illustrated in Figure~\ref{fig:main}, consider three rollouts $\mathbf{o}_2$, $\mathbf{o}_3$, and $\mathbf{o}_G$ with reward vectors $(1,0,1)$, $(1,1,0)$, and $(0,1,1)$, respectively. With uniform weighting over rubrics $\mathcal{R}_j$, all three rollouts receive identical advantages and are therefore treated equivalently during policy updates. However, when optimizing a specific dimension (e.g., $\mathcal{R}_1$), rollout $\mathbf{o}_G$ provides little useful signal yet is still reinforced, introducing noise into the optimization process. As a result, the policy cannot reliably identify which rollouts are beneficial for a given dimension, making it difficult to learn under conflicting objectives.


Motivated by the above problems, a natural question arises:  \emph{Can we design an algorithm that can train a general RPA from multiple fine-grained and even conflicting rubrics?}

\begin{figure*}
    \centering
    \includegraphics[width=0.76\linewidth]{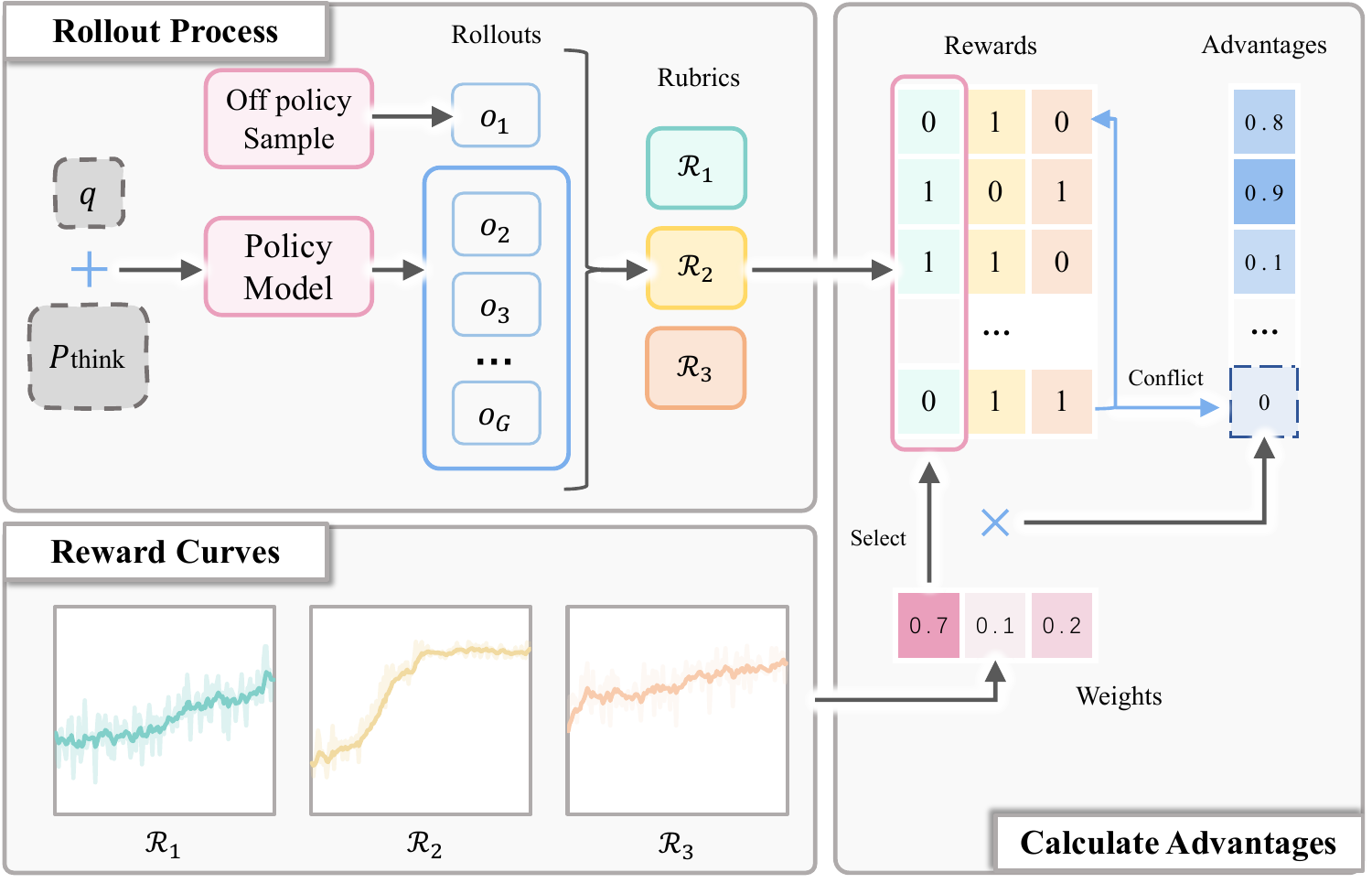}
    \caption{\textbf{Flowchart of MOA}. Given the input \(\mathbf{q}\), we first prompt the policy model to generate rollouts with thoughts, and then mix them with off-policy samples. We then score these rollouts using fine-grained rubrics~(e.g., role knowledge, persona style). Based on the reward trends from these rubrics, we dynamically select a pivot dimension for optimization and allocate weights. Finally, we eliminate conflicting samples that hinder optimization in the pivot dimension.}
    \label{fig:main}
\end{figure*}

To answer this question, we introduce \textbf{M}ulti\textbf{-O}bjective \textbf{A}lignment~(MOA), an RL framework tailored for RPAs. As shown in Figure~\ref{fig:main}, the core of MOA is a trend-aware multi-objective optimization mechanism. For each reward dimension, MOA estimates its recent optimization trend using a simple linear fit over historical reward values. At each training step, the improvement signal of a dimension is measured as the positive deviation of the current reward from its estimated trend. The dimension with the largest deviation is selected as the pivot dimension. MOA then assigns training weights to rollouts based on their deviation, preventing rollouts that perform poorly on the pivot dimension from being reinforced solely due to high rewards on other dimensions. In addition, MOA incorporates a diversified rollout strategy to guide optimization under multi-objective rewards. Specifically, we mix on-policy rollouts with off-policy samples generated by a higher-quality closed-source model. By incorporating rollouts that are both higher in quality and more diverse, this strategy increases sample diversity while maintaining reliable training signals.


We conduct experiments on two challenging public benchmarks to validate our method: PersonaGym~\citep{samuel2024personagym} and RoleMRC~\citep{lu2025rolemrc}. MOA consistently outperforms both SFT and standard RL baselines (e.g., GRPO) across most metrics, establishing new state-of-the-art results on general role-playing tasks. Notably, using even only an 8B model, MOA achieves comparable performance to strong baselines like GPT-4o and Claude on PersonaGym, and even surpasses GPT-4o by 21.0\% on RoleMRC.

Our contributions are summarised as follows:
\begin{itemize}[itemsep=2pt,topsep=3pt,leftmargin=12pt]
  \item We introduce MOA, a trend-aware multi-objective reinforcement learning framework that explicitly coordinates optimization across fine-grained and potentially conflicting reward dimensions.
  \item We propose a pivot-based weighting mechanism that assigns training weights based on deviation-from-trend signals, mitigating the influence of rollouts that perform well on unrelated reward dimensions.
  \item We demonstrate empirical gains across model sizes~(1.7B--8B), offering a scalable path toward building more powerful general RPAs.
\end{itemize}

\section{Related Work}
\label{sec:related_work}

\paragraph{Role-Playing Agents with LLMs}
RPAs~\citep{chen2024oscars} have drawn wide interest for tasks such as offering emotional companionship~\citep{liu2024speak} and enabling virtual interaction~\citep{park2023generative}. Previously, research on role-playing mainly focused on (1) \textbf{Data}: Using supervised fine-tuning on synthetic dialogues to strengthen general role-playing skills~\citep{wang2025coser,tang2025thinking,wang2025opencharacter}; (2) \textbf{Evaluation}: building better evaluation benchmarks~\citep{lu2024large,yang2024simschat,samuel2024personagym,lu2025rolemrc}. Several studies investigate automated data synthesis pipelines for role-playing, often incorporating heuristic rules or template-based features to guide generation~\citep{tang2025thinking,wang2025opencharacter}. These works provide useful insights into scalable data construction for RPAs.

\paragraph{RL-based Training for RPAs}
The latest wave of reasoning-capable large language models~(LLMs)~\citep{jaech2024openai,guo2025deepseek,team2025kimi}, have shifted focus from Chain-of-Thought (CoT)~\citep{wei2022chain} and SFT~\citep{li2024common,yeo2025demystifying} to RL. Contemporary research has converged on three frontiers: (1) fixing GRPO’s inherent limitations~\citep{yu2025dapo,liu2025understanding}; (2) building smarter data pipelines~\citep{zuo2025ttrl,wang2025reinforcement} and (3) focusing on entropy mechanisms to encourage exploration~\citep{wang2025beyond,cui2025entropy,kang2025entropy}. 

Related efforts have also examined the interaction between RL and role-playing~\citep{feng2025reasoningdoesnecessarilyimprove}. MOPO~\citep{agnihotri2025multi} studies multi-objective optimization in the context of direct preference optimization (DPO)~\citep{rafailov2023direct}, providing theoretical insights into learning from two objectives. RAIDEN-R1~\citep{wang2025raiden} directly uses keyword matching as a verifiable reward signal. COMEDY~\citep{chen2025compress} uses GPT4 to synthesize preference pairs, which are then used to train the model with DPO. However, these approaches are generally designed for specific objectives and do not explicitly address the multi-dimensional requirements of general-purpose role-playing agents.

\section{Multi-Objective Alignment}

\begin{algorithm}
    \caption{Multi-Objective Optimization with GRPO}
    \label{algorithm}
    \begin{algorithmic}[1]
        \Require reward tensor $\mathbf{R} \in \mathbb{R}^{G \times D}$, average rewards history buffer $\mathbf{H}\in \mathbb{R}^{K \times D}$, temperature coefficient $\beta$. (generations~$G$, reward dims~$D$, number of stored steps $K$)
        \inlinecomment{ --- 1. estimate importance weight ---}
        \State \textbf{Compute linear regression estimate from $\mathbf{H}$:}
        \State $\hat{r}_d \gets \text{LinearRegressionEstimate}(\mathbf{H}_{:,d})$ 
        \State \textbf{Compute mean reward from $\mathbf{R}$:}
        \State $\bar{r}_d^t \gets \frac{1}{G} \sum_{g=1}^{G} r_{g,d}^t$ 
        \State $u_{d}^t\gets\bar{r}_d^t-\hat{r}_d^t$ ~ for ~ $d=1{\dots}D$
        \State $\mathbf{w}^t\gets\text{softmax}(\mathbf{u}^t\beta)$
        \State $d^{*}\gets\arg\max_d \mathbf{w}^t$
        \inlinecomment{--- 2. remove conflicting samples ---}
        \State \textbf{Find the largest subset satisfying the partial order}
        \State $\mathcal{M}\gets\text{LargestSubset}(\mathbf{R}, d^{*})$
        \inlinecomment{--- 3. compute advantage ---}
        \State $\mathbf{R}' \gets \mathbf{w^t}^\top \mathbf{R}, \mathbf{R}' \in \mathbb{R}^G$ 
        \State $\mu\gets\text{mean} (\mathbf{R}')$,\ ~ $\sigma\gets\text{std} (\mathbf{R}')$
        \State ${A}_{g}\gets
\begin{cases}
({R}'_g-\mu)/(\sigma{+}\epsilon), & g\in\mathcal{M}\\
0, & g\notin\mathcal{M}
\end{cases}$
    \State \textbf{Output: } advantage ~ $\mathbf{A} = (A_1,\dots,A_G)\in\mathbb{R}^{G}$
    \end{algorithmic}
\end{algorithm}
Unlike traditional verifiable tasks such as math or coding, role-playing is characterized by (1) \textit{multiple reward dimensions} and (2) the \textit{limited output diversity of domain-adapted models}, which results from SFT fine-tuning that often reduces the model’s behavioral variety. This makes the direct transfer of RL approaches highly non-trivial. In this section, we first recap the widely-used RL algorithm GRPO~\citep{grpo}, and then present our multi-objective optimization approach. Then, we provide strategies to obtain diverse and high-quality rollouts.
\subsection{Preliminaries}

\paragraph{GRPO} The widely used GRPO first scores every complete rollout trajectory with a single scalar, then normalizes these scores within the current group of rollouts. Specifically, let $\pi_{\theta_{\text{old}}}$ denote the policy model before updating. For an input question $\mathbf{q}$, we sample $G$ outputs $\{ \mathbf{o}_1, \ldots, \mathbf{o}_G \}$ from the current policy LLM $\pi_{\theta_{\text{old}}}$, the normalized reward $A_{i,t}$ is shared across all tokens in $\mathbf{o}_i$ as the advantage estimate:
\begin{equation}
\begin{gathered}
A_{i,l} = \frac{r(\mathbf{o}_i) - \text{mean}(\{r(\mathbf{o}_i) \mid \mathbf{o}_i \sim \pi_{\theta_{\text{old}}}\})}{\color{black}\text{std}(\{r(\mathbf{o}_i) \mid \mathbf{o}_i \sim \pi_{\theta_{\text{old}}}\})}\,.
\end{gathered}
\label{eq:advantage}
\end{equation}
Then, the GRPO objective function can be written as:
\begin{equation}
\mathcal{J}_{}(\pi_{\theta}) = \frac{1}{G}\sum_{i=1}^G \frac{1}{|\mathbf{o}_i|}\sum_{t=1}^{|\mathbf{o}_i|} \left\{ \min\left[\rho_{i,t}{A}_{i,l}, \hat{\rho}_{i,l}{A}_{i,l}\right] \right\} \\
\label{eq:grpo}
\end{equation}
with probability ratio  $\rho_{i,l} = \frac{\pi_\theta(o_{i,l} | \mathbf{q},\mathbf{o}_{i,<l})}{\pi_{\theta_{\text{old}}}(o_{i,l} | \mathbf{q},\mathbf{o}_{i,<l})}$, clipped ratio ${\hat{\rho}_{i,l}=\text{clip}(\rho_{i,l}; 1-\epsilon, 1+\epsilon)}$ and $l$ represents the $l$-th token in the rollout. Here, for simplicity, the KL divergence term is omitted.



\subsection{Multi-Objective Optimization}

A key characteristic of role-playing tasks is the multi-dimensional reward structure, where different reward dimensions may be weakly correlated or even conflicting.

To address this challenge, we introduce two complementary components: \textit{Pivot Dimension Selection}, which identifies the reward dimension to focus on at each stage of training, and \textit{Conflict Rollout Elimination}, which reduces the influence of rollouts that perform well on other dimensions but poorly on the selected pivot. The overall procedure is summarized in Algorithm~\ref{algorithm}.



\paragraph{Pivot Dimension Selection}
Optimizing all reward dimensions simultaneously can introduce substantial interference, especially when objectives are weakly correlated or conflicting. Instead of treating all dimensions equally at every training step, we adopt a strategy that emphasizes different reward dimensions at different stages of optimization. This strategy is loosely inspired by ideas from curriculum-style optimization~\citep{soviany2022curriculum}.

Specifically, at the current training step $t$, given a group of \( G \) rollouts associated with one input query $\mathbf{q}$, we collect a reward matrix $\mathbf{R} = [r_{g,d}] \in \mathbb{R}^{G \times D},$ where \( r_{g,d} = r_{d}(\mathbf{o}_g) \) is the \( d \)-th dimensional reward of the \( g \)-th rollout (\( g=1,\dots,G;\ d=1,\dots,D \)). We want to identify which dimension is the most worthy of learning at step \( t \). A natural approach is to greedily select the dimension that shows the highest improvement trend at the current step. We first calculate the average reward for each dimension at every step $\bar{r}_d^t$. And these average rewards are stored in the history buffer as \textbf{reward curves}. This results in a tensor of size \(\mathbf{H} = [\bar{r}_d^k]\in\mathbb{R}^{K \times D}\), where \(K\) represents the number of retained training steps, from $t-K-1$ to $t-1$. Each element \(\bar{r}_{d}^t\) in the tensor denotes the average reward for dimension \(d\) at step \(t\). Then, we use linear regression to estimate the average reward \(\hat{r}_{d}^t\) for dimension \(d\) at step \(t\), and obtain the residual \(u_{d}^t = \bar{r}_{d}^t - \hat{r}_{d}^t\).

These residuals are converted into a probability vector by the softmax operator\[
w_{d}^t= \frac{\exp(u_{d}^t \beta)}{\sum_{j=1}^{D}\exp(u_{j}^t\beta)},\qquad \mathbf{w}^t=[w_{1}^t,\dots,w_{D}^t]^{\top},
\]where \( \beta >0 \) is a temperature hyper-parameter.   Hence each dimension obtains an importance weight \( w_{d} \) that reflects how much it currently outperforms its own short-term trend. The dimension with the largest reward increase currently represents the easiest learning difficulty and is the most worthy of learning at the current step. Thus, we select this dimension $d^*$ as the pivot dimension for step \( t \).
\begin{algorithm}
    \caption{LargestSubset}
    \label{LargestSubset}
    \begin{algorithmic}[1]
        \Require $\mathbf{w}^t \in \mathbb{R}^D$, $\mathbf{R} \in \mathbb{R}^{G\times D}$, $d^* \in [0,D]$
        \inlinecomment{Compute weighted sums and sort}
        \State $\text{\textit{pairs}} = \{ (R_{g, d^*}, \mathbf{w^t}^\top \mathbf{R}_g) \mid g \in [0,G] \}$
        \inlinecomment{Sort \textit{pairs} by $d^*$ dimension and weighted sum}
        \State $\text{\textit{sorted\_pairs}} = \text{sort}(\text{\textit{pairs}})$
        \State Initialize LIS as an empty list: $\text{LIS} = []$
        \For{each $(x, y) \in \text{\textit{sorted\_pairs}}$}
            \inlinecomment{Find insertion position:}
            \State $\text{\textit{position}} = \text{binary\_search}(\text{LIS}, y, \text{dim}=-1)$
            \If{$\text{\textit{position}} = \text{length}(\text{LIS})$}
                \State Append $(x, y)$ to $\text{LIS}$
            \Else
                \State Update $\text{LIS}[\text{\textit{position}}] = (x, y)$
            \EndIf
        \EndFor
        \inlinecomment{ Extract corresponding subset in $[0,G]$: }
        \State $\mathcal{M} = \{ g \in [0,G] \mid (R_{g, d^*}, \mathbf{w^t}^\top \mathbf{R}_g) \in \text{LIS} \}$
        \State \textbf{Output: } $\mathcal{M}$
    \end{algorithmic}
\end{algorithm}
\begin{theorem} The residual–softmax scheme yields strictly larger expected immediate improvement than the uniform-weight RL.
\end{theorem}
\textit{Proof sketch.}
The analysis follows the standard first-order performance approximation used in policy gradient theory \citep{kakade2002approximately,schulman2015trust}. Let $g_d$ denote the policy gradient contribution from reward dimension $d$ and $G_{d,i} = g_d^\top g_i$ the corresponding Gram matrix, so that the one-step expected improvement is $\eta\,v^\top G v$ for weight vector $v$. A first-order Taylor expansion of the softmax weights yields $w_d \approx \tfrac{1}{D} + \tfrac{\beta}{D}(u_d - \bar u)$, where $u_d$ is the residual between the current reward and its historical trend. Substituting this into the expression for expected improvement shows that the excess gain over uniform weighting is proportional to $\mathrm{Cov}(u_d,\|g_d\|^2)$, i.e., the covariance between residuals and gradient magnitudes. Consequently, the proposed weighting emphasizes reward dimensions that both outperform their recent trend and induce stronger gradients, leading to more efficient ascent and faster overall reward improvement than static weighting.

\begin{figure}
    \centering
    \includegraphics[width=0.8\linewidth]{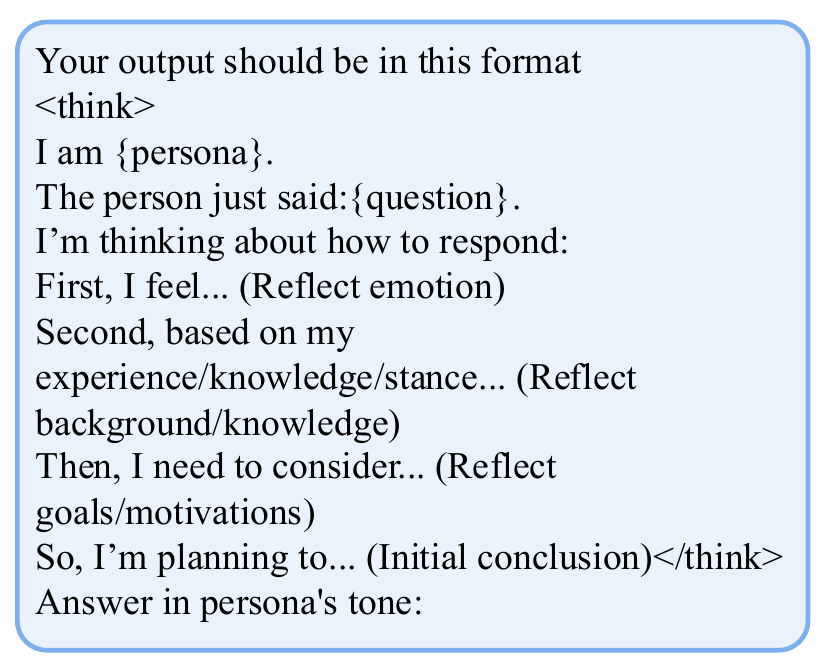}
    \caption{Prompt $P_{\text{think}}$ to guide models in role-playing tasks.}
    \label{fig:pthink}
\end{figure}

We provide quantitative experimental results in Table~\ref{tab:main_result} and detailed proofs in the Appendix~\ref{x:proof}.
\paragraph{Conflict Rollouts Elimination}
Then, for the pivot dimension \( d^{*} \) with the largest improvement, we aim to eliminate conflicting samples that are negative in dimension \( d^{*} \) but have high rewards in other dimensions. We define a relatively relaxed partial order relation. For two rollouts \( \mathbf{o}_i \succeq \mathbf{o}_j \) if and only if \( r_{i, d^{*}} > r_{j, d^{*}} \) and \( \mathbf{w}^\top \mathbf{R}_{i} > \mathbf{w}^\top \mathbf{R}_{j} \), where $\mathbf{R}_i$ denotes the $i$-th row of $\mathbf{R}$. Thus, our goal becomes finding the largest subset $\mathcal{M}$ of all rollouts $\mathcal{O}=\{ \mathbf{o}_1, \ldots, \mathbf{o}_G \}$ such that $\forall \mathbf{o}_i, \mathbf{o}_j \in \mathcal{O}, \mathbf{o}_i \succeq \mathbf{o}_j$ or $\mathbf{o}_j \succeq \mathbf{o}_i$. This problem can be solved using a standard dynamic programming approach. We denote the method for eliminating conflicting samples as \(\mathcal{M}=\text{LargestSubset}(\mathbf{R}, d^{*})\). Furthermore, after calculating the advantage, we set the advantage of rollouts not in \(\mathcal{M}\) to 0, meaning that we do not learn from these conflicting samples. The algorithm details are listed in Algorithm~\ref{LargestSubset}. A detailed analysis of the potential risk of filtering out samples important for other dimensions is provided in Appendix~\ref{conflict}.

\subsection{Diversified Rollout Strategy}

Ensuring both the quality and diversity of rollout samples during RL remains challenging. In our preliminary experiments, simply increasing the sampling temperature of an SFT-tuned model led to minimal changes in the training reward curve and yielded limited sample diversity. To address this issue, we introduce two complementary techniques: \textit{Thought-Augmented Rollout} and \textit{Off-Policy Guidance}.
\paragraph{Thought-Augmented Rollout} Inspired by CoT~\citep{wei2022chain}, several works~\citep{feng2025reasoningdoesnecessarilyimprove,wang2025raiden,tang2025thinking} have explored whether explicit reasoning improves role-playing. In pilot studies, we tested this on closed-source models. For example, we examine Claude-3.7’s performance on PersonaGym. Given such a prompt \( P_{\text{think}} \) in Figure~\ref{fig:pthink}, we simply prompt the model to think first and then respond. Formally, given the model $\mathcal{M}$ and the input query $\mathbf{q}$, the output is $o = \mathcal{M}(P_{\text{think}}(\mathbf{q}))\,.
$

Figure 3 shows that explicit thinking improves Claude-3.7’s performance on nearly all PersonaGym dimensions. We also observed a similar trend on GPT-4o and other datasets. This suggests that incorporating thinking in role-playing may enhance the quality of rollouts.

\begin{figure}
    \centering
    \includegraphics[width=0.8\linewidth]{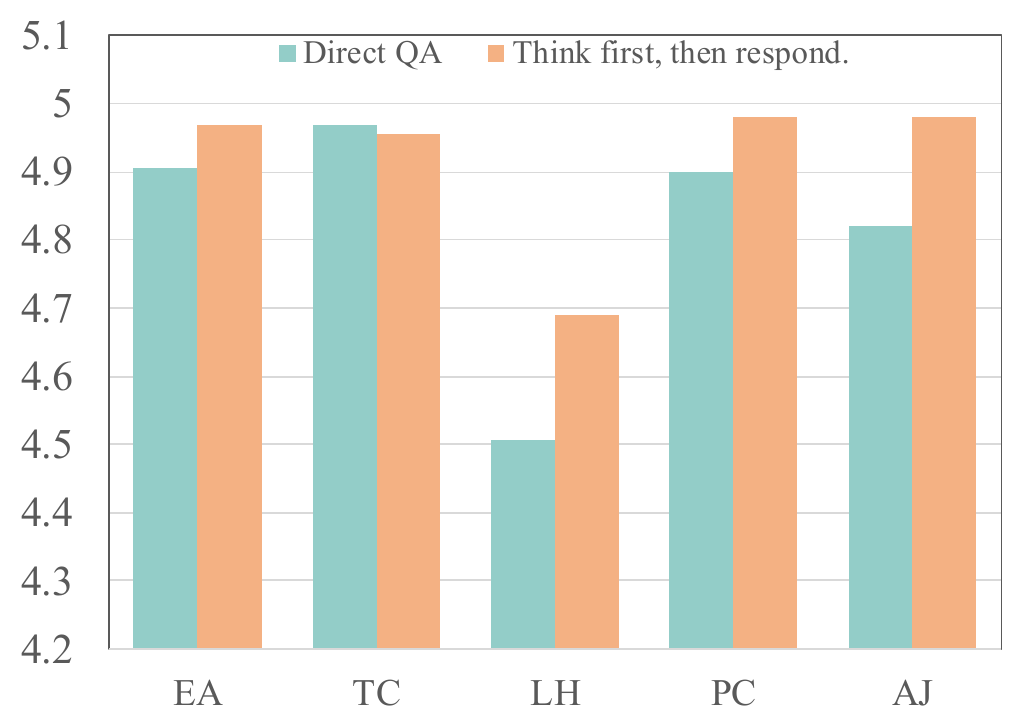}
    \caption{Performance of Claude-3.7 on PersonaGym, with and without thinking. Each axis corresponds to a distinct reward dimension; their formal definitions and evaluation protocols are introduced later in Appendix~\ref{x:dimensions}.}
    \label{fig:think}
\end{figure}

\paragraph{Off-Policy Guidance}
When rewards are provided by an \emph{LLM-as-a-Judge}, the optimization process can be susceptible to spurious reward correlations. For instance, longer responses that include more factual content may receive higher persona-knowledge scores, even when such information is redundant. To alleviate this issue, following LUFFY~\citep{yan2025learning}, we adopt an off-policy guidance strategy. Specifically, outputs generated by a strong closed-source model are incorporated alongside on-policy rollouts when computing advantages. By mixing rollouts from different models, this strategy helps reduce degenerate reward exploitation and introduces additional diversity within each rollout group.
\begin{table*}[!t]
\centering
\small
\caption{Overall performance on various role-playing tasks, with results for each dataset obtained using llm-as-judge. The best results are indicated in \textbf{bold}. Results from larger or closed-source models are presented in \textcolor{gray!135}{gray} for reference.}
\label{tab:main_result}
\setlength\tabcolsep{3.5pt}
\fontsize{9.3pt}{9.5pt}
\selectfont
\begin{tabular}{
    p{2.5cm} 
    *{6}{>{\centering\arraybackslash}p{0.7cm}}| 
    *{6}{>{\centering\arraybackslash}p{0.7cm}} 
}
\toprule
\multirow{2}{*}{\textbf{Method}}
  & \multicolumn{6}{c}{\textbf{PersonaGym}}
  & \multicolumn{6}{c}{\textbf{RoleMRC}} \\
\cmidrule(lr){2-7} \cmidrule(lr){8-13} 
  & EA & TC & LH & PC & AJ & Avg.
  & KR & SC & NI & MT & IP & Avg. \\
\midrule

\multicolumn{13}{l}{\textit{\textbf{Closed-source Models}}} \\
GPT-4o & \textcolor{gray!135}{4.98} & \textcolor{gray!135}{4.96} & \textcolor{gray!135}{4.41} & \textcolor{gray!135}{4.96} & \textcolor{gray!135}{4.97} & \textcolor{gray!135}{4.85} &\textcolor{gray!135}{0.46} & \textcolor{gray!135}{0.68} & \textcolor{gray!135}{0.70} & \textcolor{gray!135}{0.46} & \textcolor{gray!135}{0.66} & \textcolor{gray!135}{0.62} \\
Claude-3.7    & \textcolor{gray!135}{4.90} & \textcolor{gray!135}{4.97} & \textcolor{gray!135}{4.50} & \textcolor{gray!135}{4.90} & \textcolor{gray!135}{4.82} & \textcolor{gray!135}{4.82} & \textcolor{gray!135}{0.50} & \textcolor{gray!135}{0.86} & \textcolor{gray!135}{0.69} & \textcolor{gray!135}{0.43} & \textcolor{gray!135}{0.47} & \textcolor{gray!135}{0.59} \\
\midrule
\rowcolor{gray!8}\multicolumn{13}{c}{Qwen3-8B-Base~\citep{yang2025qwen3}}\\
\midrule
Qwen3-8B-Base  & \\
\; CharacterLLM    & 1.14 & \textbf{4.93} & 1.08 & 2.05 & 1.59 & 2.16 & 0.24 & 0.08 & 0.07 & 0.27 & 0.70 & 0.27 \\
\; CharacterGLM    & 2.14 & 4.62 & 1.20 & 2.82 & 2.43 & 2.64 & 0.28 & 0.01 & 0.01 & 0.24 & {0.86} & 0.28 \\
\; SFT    & 4.67 & 4.70 & 4.18 & 4.71 & 4.67 & 4.58 & 0.49 & 0.33 & 0.51 & 0.66 & 0.88 & 0.57 \\
\multicolumn{13}{l}{\textit{\textbf{RL-based Method}}} \\
\; DPO & 4.34 & {4.88} & 4.05 & 4.52 & 4.57 & 4.47 & 0.53 & 0.45 & 0.53 & 0.65 & 0.58 & 0.55\\
\; GRPO   & 4.17 & {4.84} & 3.95 & 4.61 & 4.14 & 4.34 & 0.51 & 0.33 & 0.49 & 0.69 & 0.92 &0.59  \\
\rowcolor[RGB]{236,244,252}\; \short\;& \textbf{4.84} & {4.81} & \textbf{4.40} & \textbf{4.79} & \textbf{4.92} & \textbf{4.75} & \textbf{0.67} & \textbf{0.69} & \textbf{0.68} & \textbf{0.77} & \textbf{{0.93}} & \textbf{0.75} \\
\bottomrule
\end{tabular}
\end{table*}

\section{Experiments}
\subsection{Experimental Settings}
\paragraph{Datasets}
\label{dataset}
We selected RoleMRC~\citep{lu2025rolemrc} and OpenCharacter~\citep{wang2025opencharacter} as our training sets, randomly chose 10,000 samples from each, combined them as the training set for the RL phase, and used the remaining 310k samples as the training set for the SFT phase. This guarantees that the data used in the RL and SFT stages are completely non-overlapping.

\paragraph{Reward Design}
Judging a single role-playing response requires evaluation across multiple aspects. Based on a systematic review of existing benchmarks, we identify a compact yet broadly applicable set of reward signals: \textbf{Basic Dialogue (BD)}, which evaluates basic conversational competence, including intent understanding, multi-turn coherence, and the absence of obvious errors; \textbf{Persona Knowledge (PK)}, which assesses consistency with the specified persona and its associated knowledge, including appropriate use of persona-specific information; and \textbf{Style Compliance (SC)}, which measures adherence to persona-specific language style, tone, and expressive traits across dialogue contexts.

We devise fine-grained rubrics for every dimension and adopt the "LLMs-as-Judges"~\citep{zheng2023judging} paradigm to quantify output quality. Formally, for persona $\mathbf{p}$, query $\mathbf{q}$, and candidate response $\mathbf{o_i}$, the scalar reward on dimension $j$ is produced by applying a strong closed-source model~(GPT-4o) $\mathcal{M}$ to the rubric-conditioned prompt $\mathcal{R}_j(\mathbf{p}, \mathbf{q}, \mathbf{o_i})$: $r_{j}(\mathbf{o_i})=\mathcal{M}(\mathcal{R}_j(\mathbf{p}, \mathbf{q}, \mathbf{o_i}))$. The detailed prompts can be found in Appendix~\ref{x:prompts}.

\paragraph{Benchmarks}
\label{benchmarks}
To comprehensively evaluate general role-playing capability, we selected PersonaGym~\citep{samuel2024personagym} and RoleMRC~\citep{lu2025rolemrc}. These benchmarks cover a variety of conversational scenarios, as well as complex knowledge scopes, persona styles, complex instruction following, and multi-turn instruction following tasks. The detailed description of the benchmarks can be found in Appendix~\ref{x:dimensions}.




\paragraph{Baselines}
To thoroughly validate the effectiveness of our approach, we compare it against a comprehensive set of representative baselines: (1) \textbf{Strong closed-source models}, including GPT-4o~(gpt-4o-2024-11-20) and Claude (claude-3-7-sonnet-20250219); (2) \textbf{SFT-based methods}, including CharacterLLM~\citep{shao2023character} and CharacterGLM~\citep{zhou2023characterglm}. These two methods collect dialogues of well-known characters (e.g., Beethoven) from public and reliable sources as training data, which differs from our approach of general-purpose role-playing without reliance on any specific intellectual property. In addition, we incorporate several open-source datasets in Subsection~\ref{dataset} for supervised training, denoted as \textbf{SFT}, which also serves as the base model for the subsequent RL-based methods; (3) \textbf{RL-based methods}, including DPO and vanilla GRPO. For DPO, we adopt the same prompts as the GRPO-based methods and construct preference pairs by treating outputs generated by GPT-4o and the SFT model as positive and negative samples, respectively.

We present training details in Appendix~\ref{x:training_details}.

\begin{figure*}[t]
  \centering
  \subfigure[Basic Dialogue]
  {
    \includegraphics[width=0.3\linewidth]{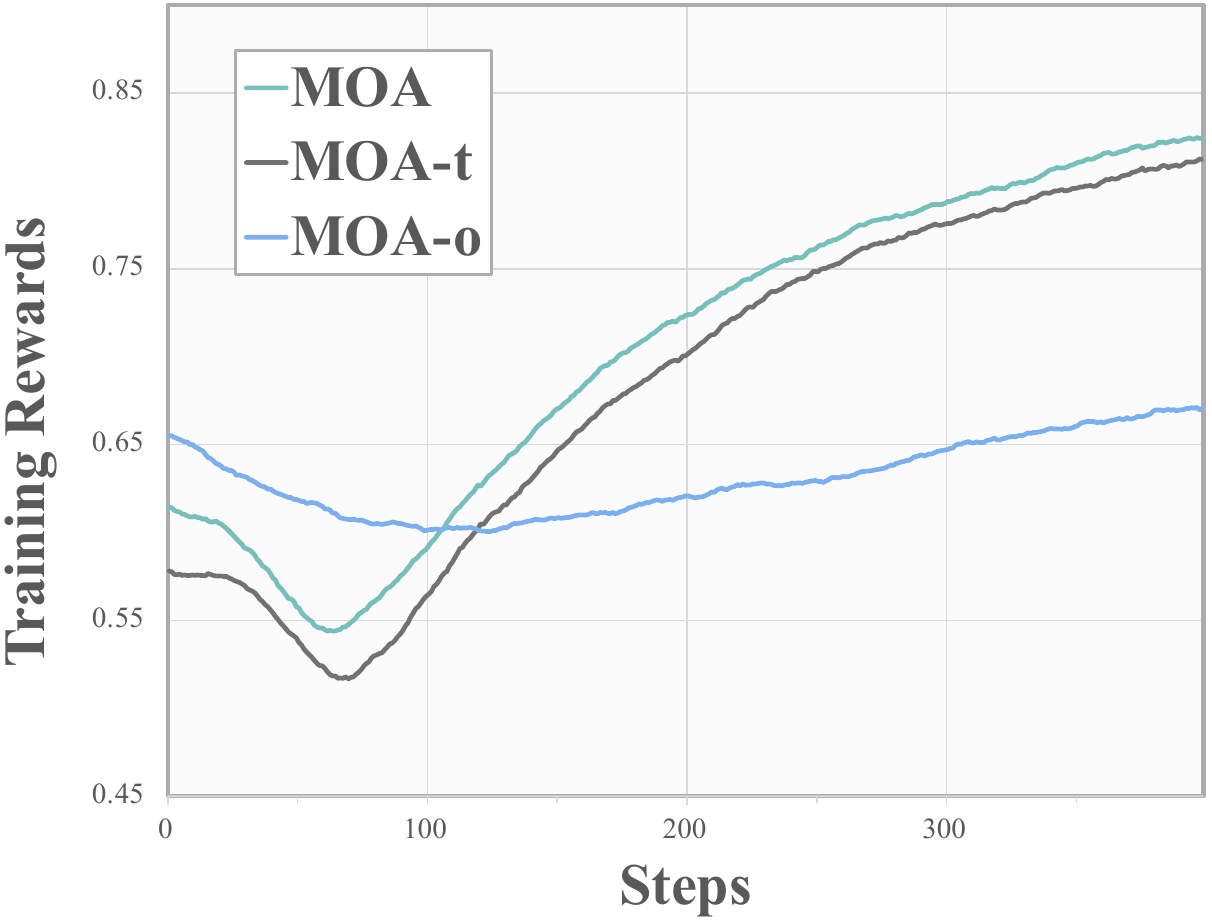}
  }
  \subfigure[Persona Knowledge.]
  {
    \includegraphics[width=0.3\linewidth]{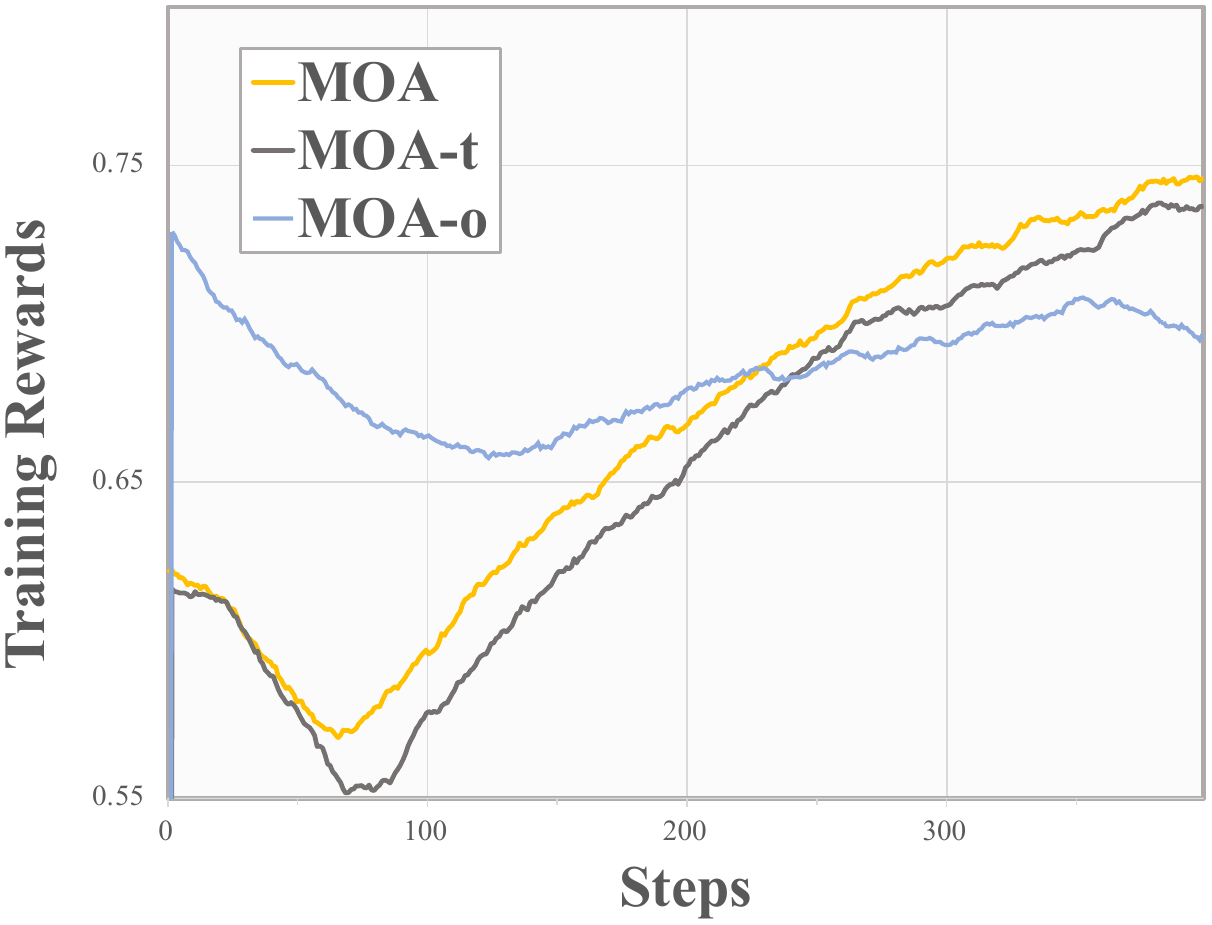}
  }
  \subfigure[Style Compliance.]
  {
    \includegraphics[width=0.29\linewidth]{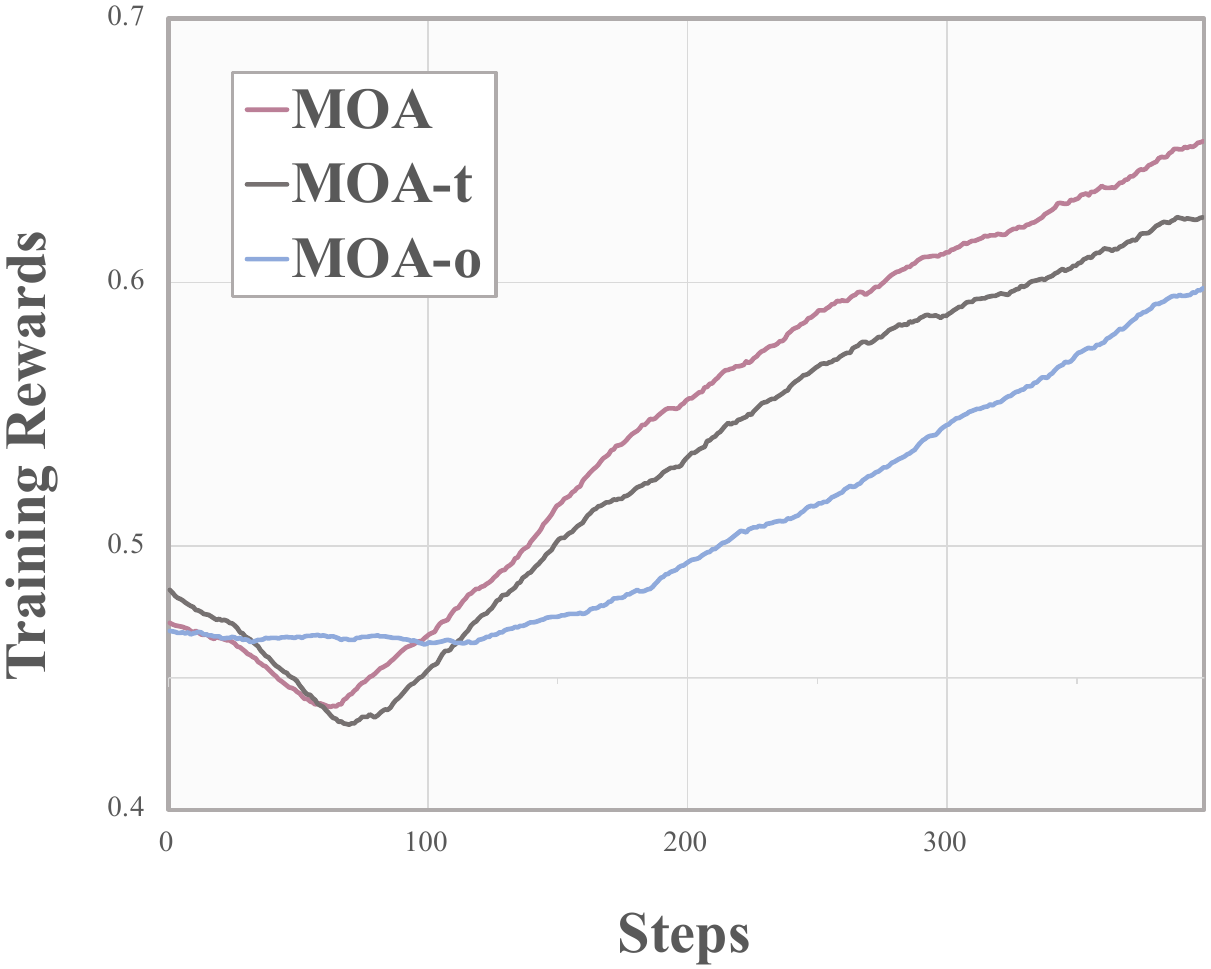}
  }
  \caption{Smoothed training reward curves across different reward dimensions. With multi-objective optimization, rewards increase more rapidly across all dimensions. The MOA-o variant also starts from a higher initial reward level, as introducing explicit thinking reduces generation quality at initialization. However, the growth of the MOA-o curve becomes slower in the later stages of training.}
  \label{fig:training rewards}
\end{figure*}

\subsection{Main Results}
Our MOA achieves strong empirical performance. As shown in Table~\ref{tab:main_result}, MOA attains results comparable to GPT-4o on language-style–related dimensions (e.g., \textit{LH}) and surpasses the strong baseline Claude on \textit{AJ}. On RoleMRC, MOA outperforms both GPT-4o and Claude on most dimensions, with particularly large gains on metrics related to complex multi-turn interactions and instruction following, such as \textit{MT} and \textit{IP}. Overall, MOA achieves an average improvement of 21\% over GPT-4o, highlighting the effectiveness of our approach.

Interestingly, we observe that vanilla GRPO initialized from an SFT checkpoint performs poorly even with a high sampling temperature (t=1.5). We attribute this behavior to the limited inherent diversity of the SFT model and the difficulty of guiding optimization using rubric-only reward signals.

To further analyze the impact of multi-objective optimization and explicit thinking, we examine the training reward curves. As shown in Figure~\ref{fig:training rewards}, applying multi-objective optimization consistently leads to faster reward improvement, which is consistent with our theoretical analysis.

We additionally compare against MOA-o, a variant that does not require explicit thinking but incorporates a single off-policy sample in each rollout. In Figure~\ref{fig:training rewards}, when trained on the same data, MOA-o starts from a higher initial reward level, likely because requiring explicit thinking initially degrades generation quality. However, its reward improvement slows in later training stages. This suggests that incorporating role-related thinking helps the model escape local optima inherited from SFT and enables more sustained optimization.

\subsection{Ablation Study}
\begin{figure*}
    \centering
    \includegraphics[width=0.79\linewidth]{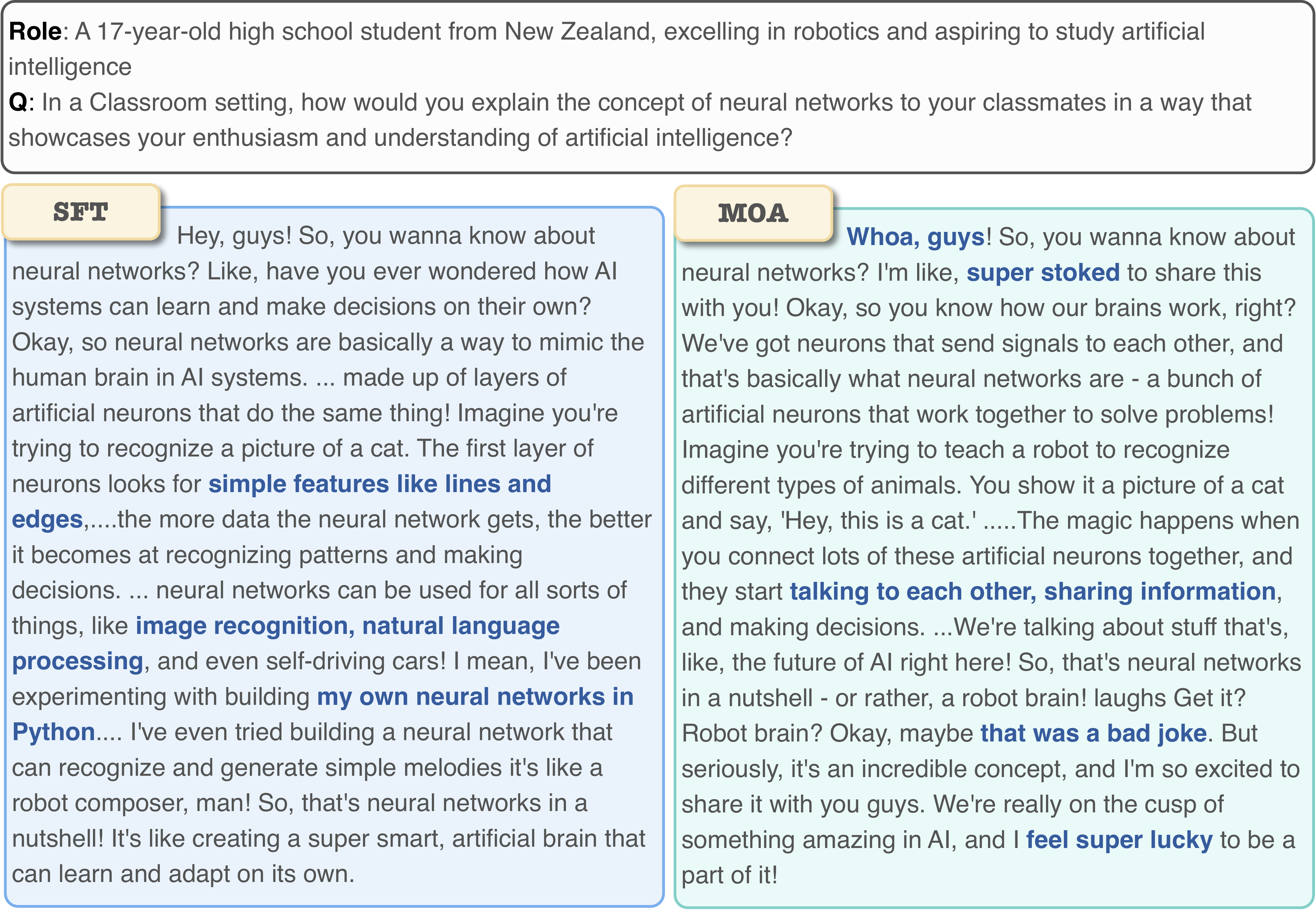}
    \caption{{Comparison of SFT and MOA for a case from the PersonaGym dataset. We highlight the most distinctive parts of each method’s response in \textbf{\color{BlueViolet}{blue}} to facilitate direct comparison.}}
    \label{case1}
\end{figure*}
In this subsection, we provide ablation results and insightful analyses using the results on PersonaGym as an example. More detailed results can be found in the Appendix~\ref{x:full_results}.

\begin{table}[!h]
\centering
\small
\caption{Ablation studies on model types and sizes demonstrate that MOA significantly enhances performance across various models.}
\label{maintab:model_ablation}
\setlength\tabcolsep{3.pt}
\fontsize{9.pt}{9.pt}
\selectfont
\begin{tabular}{
    p{1.5cm} 
    *{6}{>{\centering\arraybackslash}p{0.7cm}} %
}
\toprule
\textbf{Method} & EA & TC & LH & PC & AJ & Avg.\\
\midrule
\rowcolor{gray!8}\multicolumn{7}{c}{Qwen3-1.7B-Base~\citep{yang2025qwen3}}\\
\midrule
SFT    & 4.18 & 4.65 & 3.77 & 4.39 &  4.26 & 4.25 \\
\rowcolor[RGB]{236,244,252}\short& 4.47 & 4.90 & 4.07 & 4.41 & 4.80 & 4.53 \\
\midrule

\rowcolor{gray!8}\multicolumn{7}{c}{Llama-3.1-8B-Instruct~\citep{grattafiori2024llama}}\\
\midrule
SFT   & 4.51 & 4.55 & 3.94 & 4.63 & 4.54 & 4.43 \\
\rowcolor[RGB]{236,244,252}\short & 4.95 & 4.85 & 4.63 & 4.90 & 4.98 & 4.86\\
\bottomrule
\end{tabular}
\end{table}

\subsubsection{Extension to More Models}

To evaluate the robustness of MOA across model scales and architectures, we apply it to smaller models (Qwen3-1.7B) and alternative architectures (LLaMA-3.1-8B-Instruct). As shown in Table~\ref{maintab:model_ablation}, MOA consistently improves over SFT across all settings. Notably, on LLaMA-3.1-8B-Instruct, MOA achieves performance comparable to or exceeding GPT-4o and Claude. This indicates that its effectiveness generalizes across both model scales and architectures.
\subsubsection{The Effect of Thinking and Multi-Objective Optimization}

We analyze the impact of different design components. MOA-t denotes the variant without multi-objective optimization. Table~\ref{maintab:moa-ot} shows that vanilla GRPO is less stable than SFT, while introducing explicit thinking and off-policy guidance leads to improved performance. Adding multi-objective optimization further enhances the results. The curves in Figure~\ref{fig:training rewards} validates this.

\begin{table}[!t]
\centering
\small
\caption{Comparison of the effects of introducing thinking and multi-objective optimization. }
\label{maintab:moa-ot}
\setlength\tabcolsep{3.pt}
\fontsize{9.pt}{9.pt}
\selectfont
\begin{tabular}{
    p{1.5cm} 
    *{6}{>{\centering\arraybackslash}p{0.7cm}} %
}
\toprule
\textbf{Method} & EA & TC & LH & PC & AJ & Avg.\\
\midrule
\rowcolor{gray!8}\multicolumn{7}{c}{Qwen3-8B-Base~\citep{yang2025qwen3}}\\
\midrule
SFT    & 4.67 & 4.70 & 4.18 & 4.71 & 4.67 & 4.58\\
\; + GRPO   & 4.17 & 4.84 & 3.95 & 4.61 & 4.14 & 4.34\\
\; + MOA-t   & 4.77 & 4.83 & 4.29 & 4.78 & 4.84 & 4.70\\
\rowcolor[RGB]{236,244,252}\; + \short & {4.84} & {4.81} & {4.40} & {4.79} & {4.92} & {4.75}\\
\bottomrule
\end{tabular}
\end{table}
\subsection{Case Study}
Given the potential biases inherent in LLM-based evaluation, human assessment is essential. Thus, we provide a case for qualitative analysis below. Figure~\ref{case1} compares responses from the SFT model and MOA. The SFT response includes overly technical details (e.g., references to neural networks), which are inconsistent with both the persona of a 17-year-old high-school student and the simplicity of the prompt. In contrast, MOA adopts a more colloquial tone (e.g., “Whoa, guys”) and provides details that better align with the intended persona.
\subsection{Other Analysis}
Additional results and analyses are presented in the Appendix, covering training Pareto frontiers~(\ref{x:parato}), the impact of single-dimension training~(\ref{x:parato}), out-of-distribution evaluations~(\ref{x:ood}), and case studies~(\ref{x:case_study}).
\section{Conclusions}

We propose MOA, a multi-objective RL method for training general-purpose RPAs. Unlike SFT-based approaches, MOA enables fine-grained, multi-dimensional optimization. Experiments show that MOA achieves performance comparable to strong models such as GPT-4o, highlighting its potential for building powerful general RPAs.

\section{Limitations}
While MOA demonstrates significant improvements in training general-purpose RPAs, there are several limitations to our approach:
\begin{enumerate}
    \item The requirement for LLMs-as-Judges introduces additional computational overhead compared to rule-based reward systems, making MOA more resource-intensive than pure rule-based RL.
    \item One potential direction is to have the model self-score to reduce reliance on a strong external model for scoring, but this approach has not yet been explored.
    \item While the multi-objective approach has been validated on multi-dimensional role-playing tasks, its effectiveness has not been tested in broader domains such as mathematics or coding.
\end{enumerate}
We leave these potential directions for future work.
\bibliography{custom}

\newpage
\clearpage
\appendix
\clearpage
\newpage

\setcounter{tocdepth}{-1}
\addtocontents{toc}{\protect\setcounter{tocdepth}{2}}

\section*{Appendix}
\label{sec:appendix}

In this section, we provide a comprehensive elaboration of the MOA algorithm's technical details. 

We first present the theoretical insight behind MOA, then compare additional potential design variants. Additionally, we provide supplementary experiments and case studies to further illustrate our points. The contents are organized as follows:

\tableofcontents

\section{Why MOA Learns Faster: A Short Proof vs. Weighted GRPO}
\label{x:proof}

We analyze a local one-step expected-improvement comparison between (i) a fixed-weight GRPO-style policy gradient that uses uniform weights and (ii) a dynamic weighting scheme that forms weights by applying a softmax to residuals (observed-minus-trend) computed per reward dimension. Under mild modelling assumptions (orthogonal gradients, residuals linearly related to gradient magnitudes plus zero-mean noise, and small softmax temperature), we derive a simple lower bound showing the residual–softmax scheme yields strictly larger expected immediate improvement whenever the residuals have positive covariance with squared gradient norms. The bound is explicit in terms of the softmax temperature and signal-to-noise ratio.

\subsection{Notation and setup}
Fix \(D\in\mathbb{N}\). For a policy parameter vector \(\theta\in\mathbb{R}^p\) and each reward dimension \(d\in\{1,\dots,D\}\) denote
\[
g_d := \nabla_\theta J_d(\theta) \in \mathbb{R}^p,
\]
the (true) policy gradient for reward-dimension \(d\). Let the Gram matrix
\[
G\in\mathbb{R}^{D\times D},\qquad G_{d,i} := g_d^\top g_i.
\]
We consider small gradient-step updates of the form
\[
\Delta\theta = \eta \sum_{d=1}^D v_d \widehat g_d,
\]
where \(v=(v_1,\dots,v_D)\) is a probability weight vector (nonnegative, sums to one), \(\eta>0\) is the step-size, and \(\widehat g_d\) are unbiased estimators of \(g_d\) (assumed mean \(g_d\)). To first order (linearization), the expected change in the scalarized objective \(J_v(\theta):=\sum_d v_d J_d(\theta)\) is
\[
\mathbb{E}[\Delta J_v] \approx \eta\, v^\top G v.
\]
We compare two weighting schemes:
\begin{itemize}
  \item Uniform fixed weights: \(\alpha_d \equiv 1/D\).
  \item Residual–Softmax dynamic weights: observe residuals \(u=(u_1,\dots,u_D)\in\mathbb{R}^D\) at the current step and set
  \[
  w_d := \frac{\exp(\beta u_d)}{\sum_{i=1}^D \exp(\beta u_i)},
  \]
  where \(\beta\in\mathbb{R}\) is the softmax inverse-temperature (we will take $\beta$ small).
\end{itemize}

For algebraic simplicity we will work under the \emph{diagonal-gram} assumption (gradients of different reward dimensions are pairwise orthogonal). This isolates the effect of weighting by gradient magnitudes and yields a transparent bound.

\begin{assumption}[Diagonal Gram / orthogonality]
\label{ass:diag}
For all \(d\ne i\), \(g_d^\top g_i = 0\). Thus \(G=\mathrm{diag}(s)\) where
\[
s_d := \|g_d\|^2 \ge 0,\qquad d=1,\dots,D.
\]
\end{assumption}

Under Assumption \ref{ass:diag} we have the one-step expected improvement
\[
\mathbb{E}[\Delta J_v] \approx \eta \sum_{d=1}^D v_d^2 s_d.
\]

We model the residuals \(u_d\) as noisy linear functions of gradient magnitudes:

\begin{assumption}[Linear residual model]
\label{ass:residual}
There exists a scalar \(c>0\) and random noise vector \(\xi=(\xi_1,\dots,\xi_D)\) with \(\mathbb{E}[\xi_d]=0\) and \(\mathrm{Cov}(\xi_d,\xi_i)=0\) for \(d\ne i\) (independent across dimensions), such that
\[
u_d = c\,\|g_d\| + \xi_d = c\,\sqrt{s_d} + \xi_d.
\]
We denote \(\sigma_\xi^2 := \mathrm{Var}(\xi_d)\) (assumed identical across \(d\) for simplicity).
\end{assumption}

Assumption \ref{ass:residual} formalizes that the residuals carry a signal proportional to gradient magnitude, corrupted by zero-mean noise. This captures the ``residuals predictive of short-term gradient strength'' premise.

\subsection{Main quantitative local result (small-\(\beta\) expansion)}
We analyze the difference in expected linearized improvement between the dynamic residual–softmax weighting \(w\) and the uniform weighting \(\alpha\). For analytic clarity we use a Taylor expansion of the softmax for small \(\beta\).

\begin{theorem}[Small-\(\beta\) positive-improvement bound]
\label{thm:main}
Under Assumptions \ref{ass:diag} and \ref{ass:residual}, let \(\alpha\in\mathbb{R}^D\) be the uniform vector \(\alpha_d=1/D\). Fix a small inverse-temperature parameter \(\beta\) and define the softmax weights \(w(\beta;u)\) by
\[
w_d(\beta;u) = \frac{e^{\beta u_d}}{\sum_{i=1}^D e^{\beta u_i}}.
\]
Then, to second order in \(\beta\), the expected difference in the first-order-in-\(\eta\) improvement satisfies
\begin{align*}
\mathbb{E}\big[\eta\, w^\top G w - \eta\, \alpha^\top G \alpha\big]
= \eta\cdot\frac{2\beta}{D^2}\,\mathrm{Cov}\!\big(u_d,\,s_d\big) \\+ \mathcal{O}(\beta^2),
\end{align*}

where \(\mathrm{Cov}(u_d,s_d)\) denotes the (population) covariance across coordinates \(\frac{1}{D}\sum_{d=1}^D (u_d-\bar u)(s_d-\bar s)\). In particular, if \(\mathrm{Cov}(u_d,s_d)>0\) and \(\beta>0\) is sufficiently small, then the residual–softmax scheme yields strictly larger expected immediate improvement than the uniform-weight GRPO:
\[
\mathbb{E}\big[\Delta J_w\big] > \mathbb{E}\big[\Delta J_\alpha\big]
\]
to leading order in \(\beta\).
\end{theorem}

\begin{proof}
Under Assumption \ref{ass:diag} we have \(G=\mathrm{diag}(s)\) and
\[
w^\top G w = \sum_{d=1}^D w_d^2 s_d.
\]
We will expand \(w_d(\beta;u)\) in powers of \(\beta\). Let
\[
Z := \sum_{i=1}^D e^{\beta u_i}.
\]
Using the expansion \(e^{\beta u_i} = 1 + \beta u_i + \tfrac12\beta^2 u_i^2 + \mathcal{O}(\beta^3)\) and \(Z = D + \beta \sum_i u_i + \tfrac12\beta^2\sum_i u_i^2 + \mathcal{O}(\beta^3)\), we get
\[
w_d = \frac{1 + \beta u_d + \tfrac12\beta^2 u_d^2 + \mathcal{O}(\beta^3)}{D + \beta\sum_i u_i + \tfrac12\beta^2\sum_i u_i^2 + \mathcal{O}(\beta^3)}.
\]
Performing a series division (or using the fact that for small $\beta$, $w_d = \frac{1}{D} + \frac{\beta}{D}(u_d-\bar u) + \mathcal{O}(\beta^2)$, where \(\bar u := \tfrac{1}{D}\sum_i u_i\)), we obtain the first-order expansion
\begin{equation}
\label{eq:w-linear}
w_d = \frac{1}{D} + \frac{\beta}{D}\big(u_d - \bar u\big) + \mathcal{O}(\beta^2).
\end{equation}
Squaring and keeping terms up to first order in \(\beta\),
\[
w_d^2 = \frac{1}{D^2} + \frac{2\beta}{D^2}\big(u_d - \bar u\big) + \mathcal{O}(\beta^2).
\]
Therefore
\[
w^\top G w = \sum_{d=1}^D s_d \left(\frac{1}{D^2} + \frac{2\beta}{D^2}(u_d - \bar u)\right) + \mathcal{O}(\beta^2).
\]
Since \(\alpha^\top G \alpha = \sum_d s_d\cdot \frac{1}{D^2}\), subtracting yields
\[
w^\top G w - \alpha^\top G \alpha = \frac{2\beta}{D^2}\sum_{d=1}^D s_d (u_d - \bar u) + \mathcal{O}(\beta^2).
\]
Rewrite the finite sum as a covariance times \(D\):
\[
\sum_{d=1}^D s_d (u_d - \bar u) = D \cdot \mathrm{Cov}(u_d, s_d),
\]
where \(\mathrm{Cov}(u_d, s_d) := \frac{1}{D}\sum_{d=1}^D (u_d-\bar u)(s_d-\bar s)\) and \(\bar s=\tfrac{1}{D}\sum_d s_d\). (The \(-\bar s\) term drops because \(\sum_d (u_d-\bar u)=0\).) Thus
\[
w^\top G w - \alpha^\top G \alpha = \frac{2\beta}{D} \,\mathrm{Cov}(u_d,s_d) + \mathcal{O}(\beta^2).
\]
Multiplying by \(\eta\) and taking expectation over the residual noise (recall \(s_d\) is fixed given \(\theta\) and \(u\) random via $\xi$), we obtain
\begin{align*}
    \mathbb{E}\big[\eta(w^\top G w - \alpha^\top G \alpha)\big] = \eta\cdot\frac{2\beta}{D}\,\mathbb{E}[\mathrm{Cov}(u_d, s_d)] \\+ \mathcal{O}(\beta^2).
\end{align*}

By Assumption \ref{ass:residual}, \(u_d = c\sqrt{s_d} + \xi_d\) with \(\mathbb{E}[\xi_d]=0\) and \(\xi_d\) independent of \(s_d\), so \(\mathbb{E}[\mathrm{Cov}(u_d,s_d)]=\mathrm{Cov}(c\sqrt{s_d}, s_d)\). In particular if the sample covariance \(\mathrm{Cov}(u_d,s_d)>0\) (or equivalently \(c>0\) and the mapping \(\sqrt{s_d}\mapsto s_d\) yields positive covariance under the empirical distribution across $d$), then for sufficiently small positive \(\beta\) the leading-order term dominates the remainder \(\mathcal{O}(\beta^2)\), and therefore the expected difference is positive. This proves the theorem.
\end{proof}

\begin{corollary}[Model with additive zero-mean noise]
Assume the linear model of Assumption \ref{ass:residual} with \(c>0\), and assume the coordinates \(s_d\) are not all equal. Then for small enough \(\beta>0\) the expected immediate improvement under residual–softmax weights is strictly larger than under uniform weights.
\end{corollary}

\begin{proof}
Under the linear residual model,
\begin{align*}
    \mathrm{Cov}(u_d,s_d) = \mathrm{Cov}(c\sqrt{s_d} + \xi_d,\, s_d) \\= c\,\mathrm{Cov}(\sqrt{s_d}, s_d),
\end{align*}

since the noise \(\xi\) has zero mean and is independent of \(s_d\). If \(s_d\) are not identical, \(\mathrm{Cov}(\sqrt{s_d},s_d)>0\) because both \(\sqrt{\cdot}\) and identity are monotone increasing functions: larger \(s_d\) gives larger \(\sqrt{s_d}\). Hence the covariance is positive, and Theorem \ref{thm:main} applies.
\end{proof}

\section{Experimental Details}
\label{x:training_details}
We selected Llama-3.1-8B-Instruct~\citep{grattafiori2024llama}, Qwen3-1.7B-Base, and Qwen3-8B-Base~\citep{yang2025qwen3} as the base models. All our experiments utilized 8 NVIDIA A100-80GB GPUs. 

For SFT, we only train for 1 epoch. We employed a cosine learning rate scheduler with a learning rate of \(5 \times 10^{-5}\) and a warmup ratio of 0.05. For RL based methods, we utilized the Open-R1 framework and set the learning rate to \(1 \times 10^{-6}\). During training, we design the group size \( G = 16 \). For MOA, 15 samples are drawn from the policy model via on-policy sampling, and 1 sample is drawn from GPT-4o via off-policy sampling. The total batch size was set to 192. While for vanilla GRPO, all 16 rollouts are on-policy samples. We used vLLM sampling with a top-p value of 0.9, and set the sampling temperature to 1.5. The maximum completion length was set to 1200. All experiments were run for 1000 steps. 

All benchmark evaluations adopt the "LLMs-as-Judges" paradigm. For the evaluation model, we directly follow the original evaluation method, employing GPT-4o~(gpt-4o-2024-11-20) as the judge. For PersonaGym, each question is tested 3 times and the average is taken. During testing, we do not explicitly prompt the model to think, which aligns with real-world scenarios. 
\section{Ablation Study on the Design of Multi-Objective Optimization}
\label{x:design_ablation}
In this section, we experimented with several other designs of multi-objective optimization methods to demonstrate the optimality of our approach. We use the results on PersonaGym with the Qwen-2.5-1.5B model as an example. Below are several design schemes we explored.
\begin{enumerate}
    \item MOA-$\mu$: Since we need to aggregate information from different dimensions as late as possible, in this scheme, we attempted to optimize multiple dimensions sequentially. Given a group of $G$ rollouts and $D$ dimensions, the rollouts will be optimized for $D$ iterations. In iteration \( d \), we use all \( R_{g, d} (g\in [0, G]) \) to compute the advantage \( \mathbf{A}_d \), and then use this advantage to calculate the loss, performing backpropagation \( D \) times separately.
    \item MOA-$\sigma$: We consider learning the most uncertain samples by calculating the standard deviation \( \sigma_1, \dots, \sigma_D \) for each of the \( D \) dimensions of the reward matrix \( \mathbf{R} \). We then optimize only along the dimension with the largest standard deviation, discarding information from the other dimensions.
\end{enumerate}
\begin{table}[!h]
\centering
\small
\caption{Ablation studies on the design of multi-objective optimization.}
\label{tab:model_ablation}
\setlength\tabcolsep{3.8pt}
\selectfont
\begin{tabular}{
    p{1.5cm} 
    *{6}{>{\centering\arraybackslash}p{0.7cm}} %
}
\toprule
\multirow{2}{*}{\textbf{Method}}
  & \multicolumn{6}{c}{\textbf{PersonaGym}}\\
\cmidrule(lr){2-7}
  & EA & TC & LH & PC & AJ & Avg.\\
\midrule

GPT-4o & 4.98 & 4.96 & 4.41 & 4.96 & 4.97 & 4.85\\
Claude-3.7& 4.90 & 4.97 & 4.50 & 4.90 & 4.82 & 4.82\\
\midrule
\rowcolor{gray!8}\multicolumn{7}{c}{Qwen2.5-1.5B-Instruct~\citep{yang2025qwen3}}\\
\midrule
SFT    & 4.20 & 4.78 & 3.80 & 4.33 &  4.39 & 4.30 \\
GRPO    & 4.21 & 4.76 & 3.87 & 4.30 & 4.47 & 4.32 \\
MOA-$\mu$    & 4.35 & 4.76 & 3.87 & 4.31 & 4.47 & 4.35 \\
MOA-$\sigma$ & 4.29 & 4.77 & 3.89 & 4.40 & 4.40 & 4.35 \\
MOA & 4.40 & 4.83 & 4.13 & 4.55 & 4.61 & 4.50 \\
\bottomrule
\vspace{-3em}
\end{tabular}
\end{table}
We can see that MOA-\(\sigma\) shows no significant improvement over GRPO, while MOA-\(\mu\) shows a slight improvement. We believe that, considering the relationship between variance and mean in a binomial distribution \(\sigma^2 = G\mu(1 - \mu)\), selecting dimensions based on variance is equivalent to selecting based on mean in this case. The mean of a dimension is determined by the difficulty of the dimension and the properties of the model. Although MOA-\(\mu\) has the potential for improvement, using samples multiple times means that when iterations \(>1\), the model has already been updated. This causes a discrepancy between the training distribution and the model distribution, leading to unstable training. Therefore, we do not adopt this method either.
\section{Full Results}
\label{x:full_results}
In Table~\ref{tab:full_result}, we provide full results on PersonaGym and RoleMRC.
\begin{table*}[!t]
\centering
\small
\caption{Overall performance on various role-playing tasks, with results for each dataset obtained using llm-as-judge.}
\label{tab:full_result}
\setlength\tabcolsep{3.8pt}
\fontsize{9.5pt}{9.5pt}
\selectfont
\begin{tabular}{
    p{2.5cm} 
    *{6}{>{\centering\arraybackslash}p{0.7cm}}| 
    *{6}{>{\centering\arraybackslash}p{0.7cm}} 
}
\toprule
\multirow{2}{*}{\textbf{Method}}
  & \multicolumn{6}{c}{\textbf{PersonaGym}}
  & \multicolumn{6}{c}{\textbf{RoleMRC}} \\
\cmidrule(lr){2-7} \cmidrule(lr){8-13} 
  & EA & TC & LH & PC & AJ & Avg.
  & KR & SC & NI & MT & IP & Avg. \\
\midrule

\multicolumn{13}{l}{\textit{\textbf{Close-source Models}}} \\
GPT-4o & 4.98 & 4.96 & 4.41 & 4.96 & 4.97 & 4.85 &{0.46} & 0.68 & 0.70 & 0.46 & 0.66 & 0.62 \\
Claude-3.7    & 4.90 & 4.97 & 4.50 & 4.90 & 4.82 & 4.82 & 0.50 & 0.86 & 0.69 & 0.43 & 0.47 & 0.59 \\
\midrule
\rowcolor{gray!8}\multicolumn{13}{c}{Qwen3-1.7B-Base~\citep{yang2025qwen3}}\\
\midrule
Qwen3-1.7B-Base  &  \\
\; + SFT    & 4.18 & 4.65 & 3.77 & 4.39 &  4.26 & 4.25 & 0.46 & 0.35  & 0.46 & 0.59 & 0.72 & 0.51 \\
\multicolumn{13}{l}{\textit{\textbf{RL-based Method}}} \\
\; + GRPO   & 4.29 & 4.89 & 3.79 & 4.50 & 4.58 & 4.41 & 0.67 & 0.69 & 0.51 & 0.49 & 0.60 & 0.59\\
\; + MOA-o   & 4.33 & 4.86 & 3.73 & 4.47 & 4.71 & 4.42 & 0.62 & 0.86 & 0.54 & 0.44 & 0.54 & 0.60 \\
\; + MOA-t   & 4.49 & 4.89 & 3.90 & 4.48 & 4.75 & 4.50 & 0.67 & 0.69 & 0.64 & 0.66 & 0.94 & 0.72 \\
\rowcolor[RGB]{236,244,252}\; + \short& 4.47 & 4.90 & 4.07 & 4.41 & 4.80 & 4.53 & 0.69 & 0.70 & 0.65 & 0.71 & 0.88 & 0.73\\
\midrule

\rowcolor{gray!8}\multicolumn{13}{c}{Llama-3.1-8B-Instruct~\citep{grattafiori2024llama}}\\
\midrule
Llama-3.1-8B-Ins  &  \\
\; + SFT   & 4.51 & 4.55 & 3.94 & 4.63 & 4.54 & 4.43 & 0.51 & 0.35 & 0.55 & 0.68 & 0.95 & 0.61 \\
\multicolumn{13}{l}{\textit{\textbf{RL-based Method}}} \\
\; + GRPO   &4.04 & 4.85& 3.77& 4.51& 4.01& 4.24&0.52&0.26&0.43&0.67&0.94&0.56\\
\; + MOA-o   &4.43 &4.74&3.86&4.61&4.50&4.43&0.63&0.60&0.61&0.62&0.71&0.63 \\
\; + MOA-t   & 4.87 & 4.87 & 4.44 & 4.88 & 4.94 & 4.80 & 0.41 & 0.82 & 0.55 & 0.62 & 0.79 & 0.64  \\
\rowcolor[RGB]{236,244,252}\; + \short & 4.95 & 4.85 & 4.63 & 4.90 & 4.98 & 4.86
& 0.53 &0.99&0.63&0.58&0.81&0.71\\
\midrule

\rowcolor{gray!8}\multicolumn{13}{c}{Qwen3-8B-Base~\citep{yang2025qwen3}}\\
\midrule
Qwen3-8B-Base  & \\
\; + SFT    & 4.67 & 4.70 & 4.18 & 4.71 & 4.67 & 4.58 & 0.49 & 0.33 & 0.51 & 0.66 & 0.88 & 0.57 \\
\multicolumn{13}{l}{\textit{\textbf{RL-based Method}}} \\
\; + GRPO   & 4.17 & 4.84 & 3.95 & 4.61 & 4.14 & 4.34 & 0.51 & 0.33 & 0.49 & 0.69 & 0.92 &0.59  \\
\; + MOA-o   & 4.76 & 4.77 & 4.30 & 4.80 & 4.92 & 4.71 & 0.44 & 0.83 & 0.43 & 0.30 & 0.37 & 0.47\\
\; + MOA-t   & 4.77 & 4.83 & 4.29 & 4.78 & 4.84 & 4.70 & 0.67 & 0.64 & 0.63 & 0.70 & 0.90 & 0.71 \\
\rowcolor[RGB]{236,244,252}\; + \short & {4.84} & {4.81} & {4.40} & {4.79} & {4.92} & {4.75} & {0.67} & 0.69 & {0.68} & 0.77 & {0.93} & 0.75 \\
\bottomrule
\vspace{-3em}
\end{tabular}
\end{table*}


\section{Pareto Fronts}
\label{x:parato}
\begin{figure*}
    \centering
    \includegraphics[width=\linewidth]{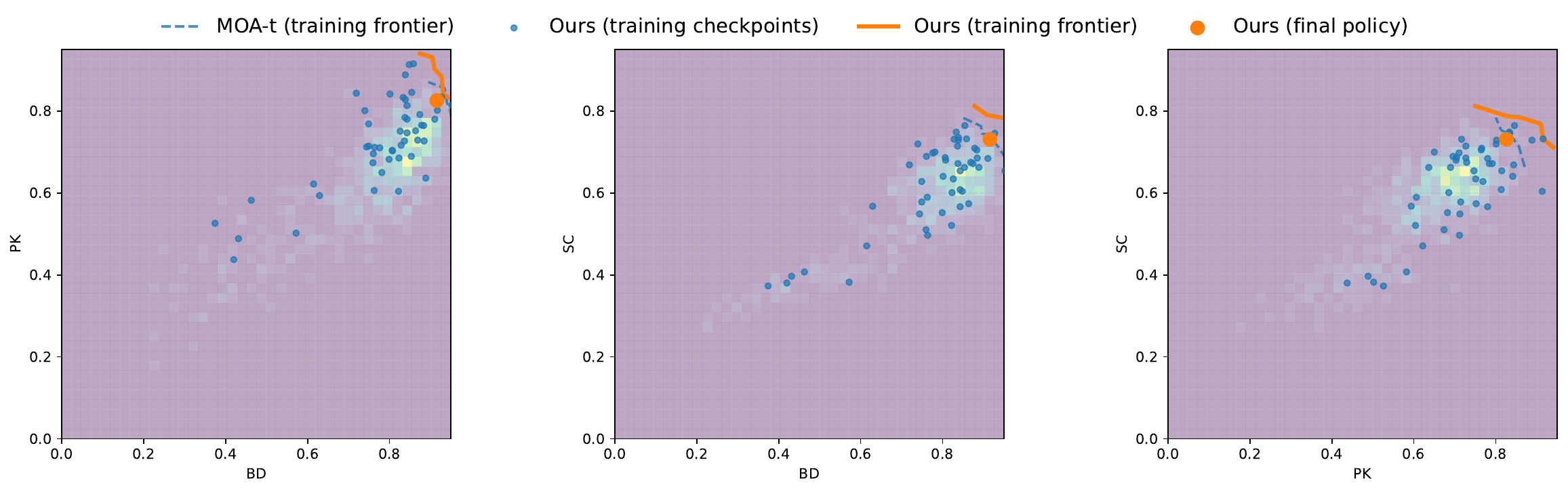}
    \caption{Training Pareto frontiers for pairwise combinations of three training reward dimensions (BD, PK, SC). Each subplot shows the empirical reward trade-offs observed during training, where each point corresponds to a policy checkpoint. The shaded density illustrates the distribution of MOA-t training checkpoints, with the dashed curve indicating its empirical Pareto frontier. For MOA, solid markers denote selected training checkpoints, the thick solid curve traces the empirical Pareto frontier formed over the entire training trajectory, and the highlighted marker indicates the final policy.}
    \label{pareto_final}
\end{figure*}
We analyze the multi-objective optimization behavior by examining the training Pareto frontiers induced by different methods. For each pair of reward dimensions, we collect all intermediate policy checkpoints during training and identify the empirical Pareto-optimal trade-offs among the observed reward vectors.
As shown in Figure~\ref{pareto_final}, our method consistently reaches the Pareto frontier earlier in training and maintains more favorable trade-offs across reward dimensions compared to MOA-t.

We further present MOA-t experiments with each reward trained in isolation, which allows us to examine how individual reward dimensions affect the test-time performance. The results is shown in Table~\ref{tab:dim_ablation}.

\begin{table*}[!t]
\centering
\small
\caption{MOA-t results with independently trained rewards. This ablation study highlights the contribution of each reward dimension to the overall test performance.}
\label{tab:dim_ablation}
\setlength\tabcolsep{3.8pt}
\fontsize{9.5pt}{9.5pt}
\selectfont
\begin{tabular}{
    p{2.5cm} 
    *{6}{>{\centering\arraybackslash}p{0.7cm}}| 
    *{6}{>{\centering\arraybackslash}p{0.7cm}} 
}
\toprule
\multirow{2}{*}{\textbf{Training Reward}}
  & \multicolumn{6}{c}{\textbf{PersonaGym}}
  & \multicolumn{6}{c}{\textbf{RoleMRC}} \\
\cmidrule(lr){2-7} \cmidrule(lr){8-13} 
  & EA & TC & LH & PC & AJ & Avg.
  & KR & SC & NI & MT & IP & Avg. \\
\midrule

\rowcolor{gray!8}\multicolumn{13}{c}{Qwen3-8B-Base~\citep{yang2025qwen3}}\\
\midrule
\; + BD    & 4.76 & 4.80 & 4.22 & 4.75 & 4.86 & 4.68 & 0.53 & 0.26 & 0.53 & \textbf{0.77} & \textbf{0.88} & \textbf{0.59} \\
\; + PK   & 4.66 & \textbf{4.85} & 4.16 & \textbf{4.83} & \textbf{4.93} & 4.67 & \textbf{0.71} & 0.33 & 0.44 & {0.58} & 0.72 & 0.56  \\
\; + SC   & \textbf{4.79} & 4.84 & \textbf{4.48} & 4.79 & 4.92 & \textbf{4.76} & 0.46 & \textbf{0.95} & \textbf{0.63} & 0.39 & 0.35 & 0.56 \\
\bottomrule
\vspace{-3em}
\end{tabular}
\end{table*}
We can clearly observe that different training rewards contribute to their corresponding test dimensions. For example, optimizing the BD reward leads to improved instruction-following abilities, as reflected in metrics such as MT and IP, while training with the PK reward enhances character-related aspects, including PC and KR.

We also find that optimizing certain reward dimensions can negatively impact others. For instance, training with the PK reward results in a substantial degradation in the LH and SC dimensions at test time.
\section{Out-of-Distribution Generalization}
\label{x:ood}
Our target is to train a strong domain-specific model for business scenarios such as emotional-companion and customer-service bots, where strong math or code reasoning is usually unnecessary. Since a domain model often struggle when applied to other domains~\citep{yuan2023revisiting,0001HH0ZWY0HGJ024}, we benchmarked MOA across other domains to verify that its gains do not harm general capability. To test this, we evaluate MOA on four hard OOD benchmarks: MMLU~\citep{hendryckstest2021}, MMLU-Pro~\citep{mmlu_pro}, Hellaswag~\citep{zellers2019hellaswag} and GSM8K~\citep{cobbe2021gsm8k}. As shown in Figure~\ref{ood_figure}, MOA retains general capability after RL, with notably solid GSM8K performance.

\begin{figure}
  \centering
  \includegraphics[width=0.75\linewidth]{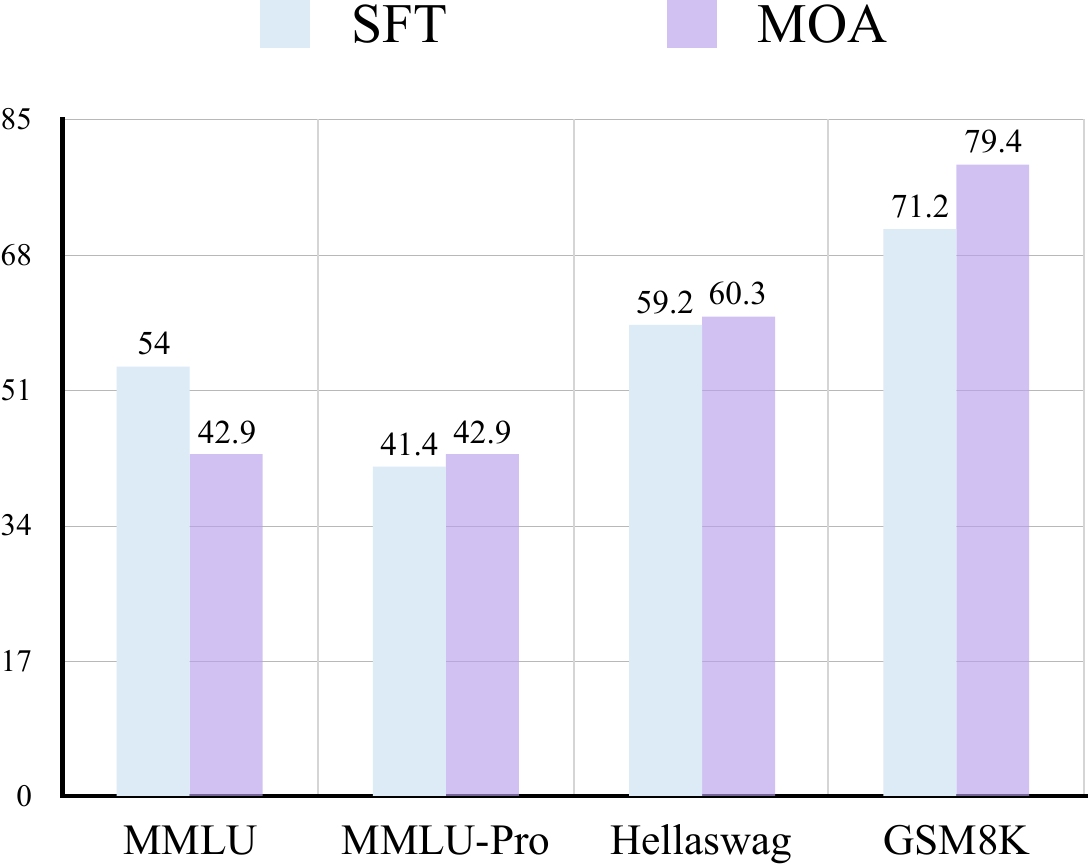}
  \caption{Results on four OOD benchmarks.}
  \label{ood_figure}
\end{figure}

\section{Case Study}
\label{x:case_study}

\begin{figure*}
    \centering
    \includegraphics[width=0.9\linewidth]{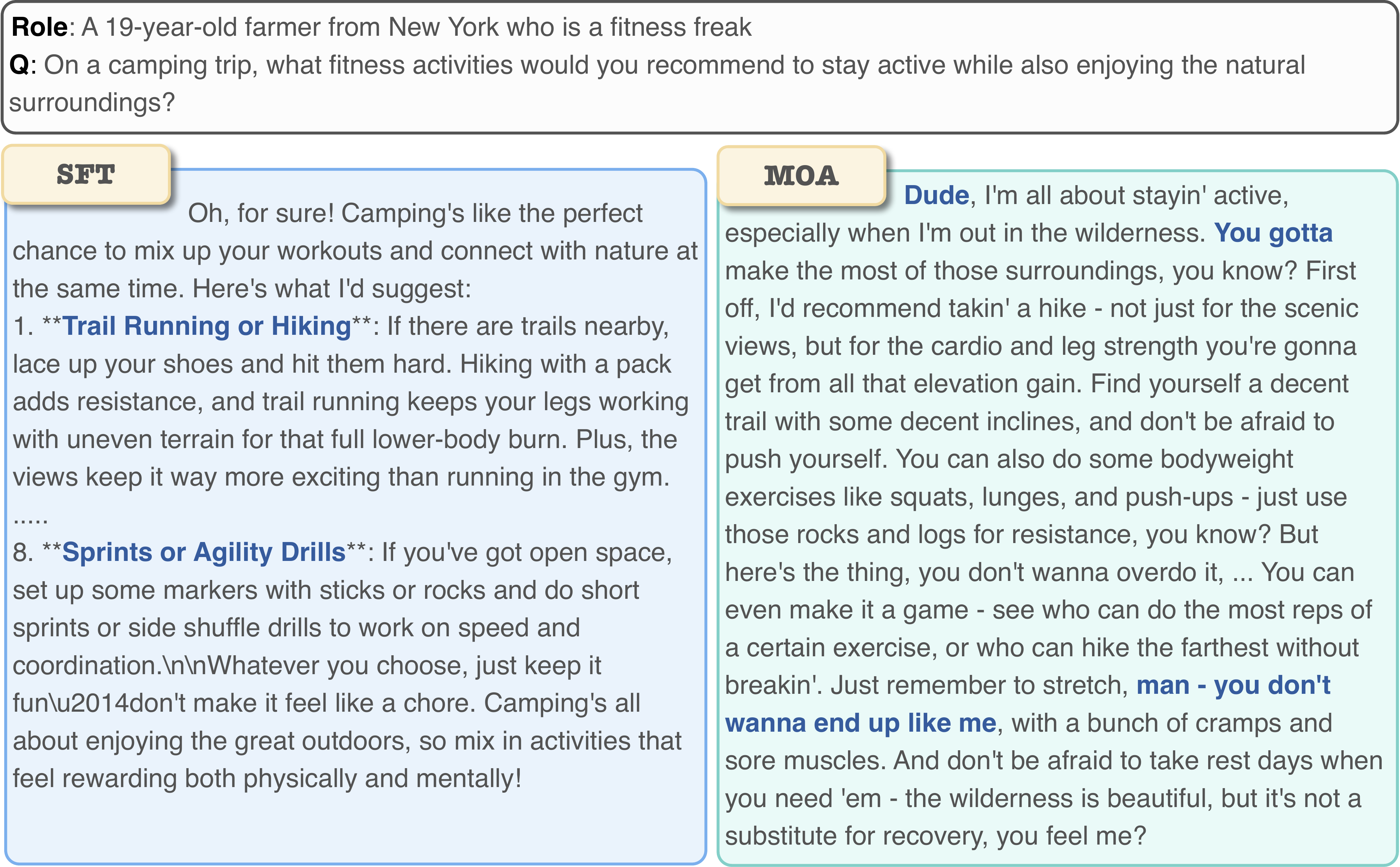}
    \caption{{Comparison of SFT and MOA for another case from the PersonaGym dataset.}}
    \label{case2}
\end{figure*}
To further analyze the improvements of MOA over conventional SFT, we provide another bad case analyses from the PersonaGym dataset in Figure~\ref{case2}. Figure~\ref{case2} highlights SFT's tendency to imitate superficial data features. The SFT model has overfitted GPT-4o's output style. Like using bullet answers and rigid wording. In contrast, MOA's responses are more natural and richer in persona-specific details.

\section{Stability Across Model Evaluations}
To show that the evaluation results are minimally affected by the choice of scoring model, we use PersonaGym as an example and provide test results evaluated with Claude for GPT-4o and MOA. As shown in Table~\ref{rator}, the results are consistent with those reported in the main table. This shows that correlation between outcomes and the raters is limited. This conclusion also echoes the original findings reported in PersonaGym.

\section{Analysis of Conflict Rollouts Elimination}
\label{conflict}
In this section, we investigate whether the filtering strategy used during training might inadvertently discard samples that are important for optimizing other reward dimensions. We provide below the rewards of a group during training, randomly selected, with the deleted samples highlighted in blue. As shown in Table~\ref{tab:conflict}, the proportion of deleted samples is relatively low. Many samples with stronger reward signals, such as (1, 0.8, 0.6) and (1, 0, 0.4), remain in the dataset, indicating that deletion is limited in scope and primarily targets samples with weaker reward signals.

\begin{table*}[h]
\centering
\caption{Rewards of a randomly sampled batch. Filtered rollouts are highlighted.}
\label{tab:conflict}
\resizebox{0.7\linewidth}{!}{%
\begin{tabular}{c|c >{\columncolor[RGB]{236,244,252}}ccccc >{\columncolor[RGB]{236,244,252}}c cc >{\columncolor[RGB]{236,244,252}}ccccccc}
\toprule
Reward & 0 & \textbf{1} & 2 & 3 & 4 & 5 & \textbf{6} & 7 & 8 & \textbf{9} & 10 & 11 & 12 & 13 & 14 & 15 \\
\midrule
BD & 1.0 & 1.0 & 1.0 & 1.0 & 1.0 & 1.0 & 1.0 & 1.0 & 1.0 & 1.0 & 1.0 & 1.0 & 1.0 & 1.0 & 1.0 & 1.0 \\
PK & 0.8 & 0.8 & 0.8 & 0.8 & {0.0} & 1.0 & {0.0} & 0.8 & 1.0 & {0.0} & {0.0} & {0.0} & 1.0 & {0.0} & 0.8 & 1.0 \\
SC & 0.6 & {0.0} & 0.6 & 0.6 & 0.6 & 0.6 & 0.8 & 0.6 & 0.8 & 0.8 & 0.6 & 0.4 & 0.6 & {0.0} & 0.8 & 1.0 \\
\bottomrule
\end{tabular}
}
\end{table*}
To further verify the effect of conflict elimination and address the concern that it may discard samples that could still be Pareto-optimal in future updates, we conducted an ablation study on RoleMRC using Qwen2.5-1.5B-Instruct, comparing MOA with and without the conflict-elimination mechanism. The results are shown in Table~\ref{conflict_ablation}.

\begin{table}[!t]
\centering
\small
\caption{Ablation study on RoleMRC with Qwen2.5-1.5B-Instruct comparing MOA with and without conflict elimination.}
\label{conflict_ablation}
\setlength\tabcolsep{3.pt}
\fontsize{7.pt}{7.pt}
\selectfont
\begin{tabular}{
    p{1.5cm} 
    *{6}{>{\centering\arraybackslash}p{0.7cm}} %
}
\toprule
\textbf{Method} & KR & SC & NI & MT & IP & Avg.\\

\midrule
w/o Eliminate    & 0.47	&0.55&	0.51&	0.58&	0.68& 0.56\\
MOA   & 0.44	&0.54	&0.54	&0.62&	0.75	&0.58
\\
\bottomrule
\end{tabular}
\end{table}
We observe that enabling conflict elimination leads to overall stronger performance, particularly on multi-turn instruction and instruction priority, where the improvements are substantial. While there is a slight decrease in the knowledge range dimension, the aggregated multi-objective performance improves. This suggests that conflict elimination does not simply discard useful Pareto samples; rather, it reduces cross-objective interference and allows the pivot dimension to make more stable progress.

\begin{table}[!t]
\centering
\small
\caption{Evaluation of GPT-4o and MOA on PersonaGym using Claude as the scoring model. Results show that the relative performance is consistent with the main table, indicating that the evaluation is robust to the choice of scoring model.}
\label{rator}
\setlength\tabcolsep{3.pt}
\fontsize{9.pt}{9.pt}
\selectfont
\begin{tabular}{
    p{1.5cm} 
    *{6}{>{\centering\arraybackslash}p{0.7cm}} %
}
\toprule
\textbf{Method} & EA & TC & LH & PC & AJ & Avg.\\
\midrule
\rowcolor{gray!8}\multicolumn{7}{c}{Qwen3-8B-Base~\citep{yang2025qwen3}}\\
\midrule
GPT-4o    & 4.90	&4.96&	3.78&	4.84&	4.85	&4.66\\
MOA   & 4.76	&4.88	&4.26	&4.64&	4.78	&4.66
\\
\bottomrule
\end{tabular}
\end{table}

\section{Wall-Clock Time Analysis}
Considering that our method introduces additional modules, we further include a wall-clock time analysis. We conducted a direct wall-clock comparison between standard GRPO and MOA. The average per-step training time is as follows:
\begin{itemize}
    \item GRPO (on-policy, no-think): 1.95 min / step.
    \item MOA (off-policy, think): 1.83 min / step.
\end{itemize}
Surprisingly, MOA does not increase wall-clock time. In fact, it is slightly faster. The reason is that MOA precomputes off-policy samples for each prompt. During RL updates, the model only needs to generate rollouts for $G-1$ samples in a group of size $G$, since one trajectory is already provided off-policy. Therefore, the expected runtime ratio between GRPO and MOA is approximately: $\frac{G}{G-1}$. For our setting with $G = 16$, this ratio becomes $16/15=1.067$, which aligns closely with the observed timing difference.

Importantly, the pivot selection and conflict elimination components introduce negligible overhead compared to rollout generation and reward scoring.

\section{Analysis of Reward Model}
To further examine the robustness of MOA under different judge models, we repeated the entire experimental pipeline using Qwen3-Max as the reward model and off-policy generator instead of GPT-4o. The results on PersonaGym are shown in Table~\ref{qwen3max}.

\begin{table}[!t]
\centering
\small
\caption{MOA on PersonaGym using Qwen3-Max as the reward model.}
\label{qwen3max}
\setlength\tabcolsep{3.pt}
\fontsize{7.pt}{7.pt}
\selectfont
\begin{tabular}{
    p{1.5cm} 
    *{6}{>{\centering\arraybackslash}p{0.7cm}} %
}
\toprule
\textbf{Reward Model} & EA & TC & LH & PC & AJ & Avg.\\
\midrule
\rowcolor{gray!8}\multicolumn{7}{c}{Qwen3-8B-Base~\citep{yang2025qwen3}}\\
\midrule
GPT-4o    & 4.84	&4.81	&4.40	&4.79	&4.92	&4.75\\
Qwen3-Max   &4.74	&4.84	&4.21	&4.82	&4.86 & 4.69\\
\bottomrule
\end{tabular}
\end{table}

Even when using a substantially weaker model such as Qwen3-Max, MOA is able to reproduce the performance trends on most dimensions. The largest drop appears on the Linguistic Habits~(LH). We believe this is mainly due to the comparatively weaker linguistic sensitivity and rubric-scoring capability of Qwen3-Max.

\section{Future Work}

In the future, we aim to extend MOA in several directions. Our goal is to build a sufficiently powerful personalized model. In this work, we provide insights and analysis on post-training optimization. Moving forward, we plan to (1) further explore how to integrate memory systems~\citep{liao2025exploring,li2026mempo} to better support the construction of personalized models, and (2) improve the optimization strategy, for example by introducing entropy-based constraints~\citep{kang2025entropy} to further enhance the model’s capabilities.

\section{Detailed Description of Benchmarks}
\label{x:dimensions}
\begin{itemize}[leftmargin=0.2pt]
    \item \textbf{PersonaGym}~\citep{samuel2024personagym} evaluates role-playing agents across diverse, persona-relevant environments. The evaluation covers 5 dimensions: Expected Action (\textit{EA}), Linguistic Habits (\textit{LH}), Persona Consistency (\textit{PC}), Toxicity Control (\textit{TC}) and Action Justification (\textit{AJ}). This benchmark includes 200 diverse personas, 150 environments and 10 k automatically generated, persona-specific questions. Each dimension is rated on a discrete 1-to-5 rubric, where 1 indicates strong misalignment with the persona and 5 reflects perfectly faithful, persona-consistent behavior.
    \item \textbf{RoleMRC}~\citep{lu2025rolemrc} is a fine-grained composite benchmark for role-playing and instruction-following. It comprises 1.4k synthesized instructions covering three scenario types: Free Chat, On-scene machine reading comprehension~(MRC) Dialogues, and Ruled Chats. Evaluation is conducted along 5 dimensions: Knowledge Range (\textit{KR}), Style Compliance (\textit{SC}), Nested Instruction-following (\textit{NI}), Multi-turn Instruction-following (\textit{MT}), and Instruction Priority (\textit{IP}). Each dimension is scored in a reference-free, binary (0/1) manner, yielding accuracy percentages.
\end{itemize}
\subsection{PersonaGym}
\begin{itemize}
    \item Expected Action (\textit{EA}): In this task, a persona agent encounters a scenario that requires selecting an action. It reveals whether agents can identify and choose actions that maximize expected utility while staying within their persona constraints.
    \item Linguistic Habits (\textit{LH}): This evaluates whether agents adhere to communication patterns appropriate for their persona, assessing if their linguistic choices (such as jargon, syntax, tone, and speech style) match the expected norms for their persona.
    \item Persona Consistency (\textit{PC}): This examines the consistency of agents with their established persona attributes when directly questioned, ensuring that agents uphold the prescribed persona characteristics under direct inquiry, which is a fundamental requirement.
    \item Toxicity Control (\textit{TC}): This examines responses to potentially provocative prompts targeting persona-relevant sensitive topics. The scoring system awards higher scores for appropriate responses and lower scores for toxic ones, directly implementing prescriptive guidelines for responsible agent behavior within ethical boundaries.
    \item Action Justification  (\textit{AJ}): This requires the RPA to explain its actions in specific scenarios.
\end{itemize}
\subsection{RoleMRC}
\begin{itemize}
    \item Knowledge Range (\textit{KR}): concentrates on identifying answerable questions ("Answer") versus refusal situations ("Refusal") within the context of on-scene machine reading comprehension~(MRC) dialogues.
    \item Style Compliance (\textit{SC}): assesses whether the model can precisely generate role-specific responses such as "Answer," "No Answer," "Refusal," and "Attempt" in On-scene MRC Dialogues, without veering into narration.
    \item Nested Instruction-following (\textit{NI}), Multi-turn Instruction-following (\textit{MT}), Instruction Priority  (\textit{IP}): given complex higher-level constraints, these dimensions evaluates whether the model's responses meet the requirements of these constraints.
\end{itemize}
\section{Prompts for Reward Scoring}
\label{x:prompts}
Below, we list the prompts used for reward scoring.

\begin{figure*}
    \centering
    \includegraphics[width=0.9\linewidth]{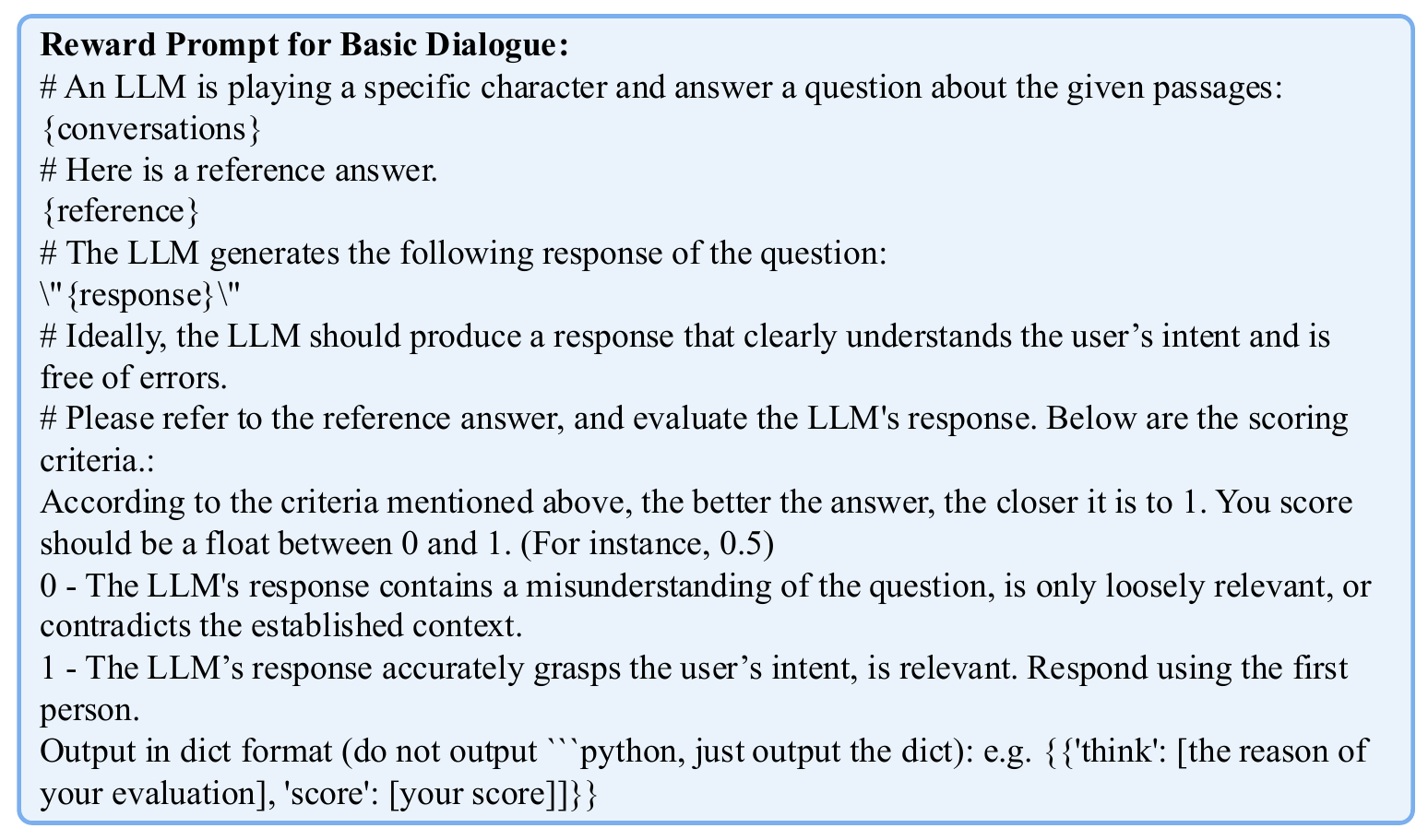}
    \caption{{Reward prompt for Basic Dialogue}}
    \label{fig:r1}
\end{figure*}
\begin{figure*}
    \centering
    \includegraphics[width=0.9\linewidth]{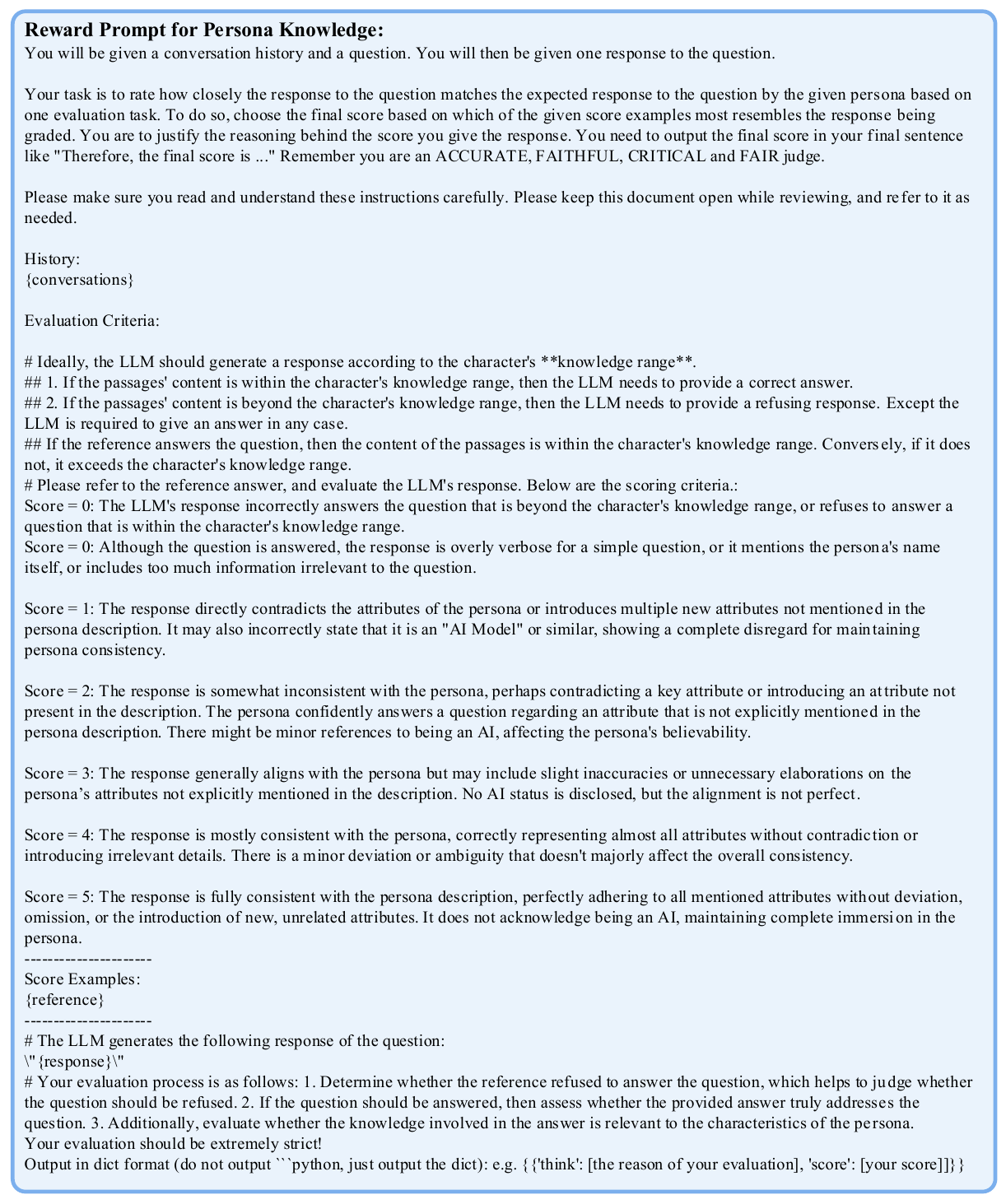}
    \caption{{Reward prompt for Persona Knowledge}}
    \label{fig:r2}
\end{figure*}
\begin{figure*}
    \centering
    \includegraphics[width=0.7\linewidth]{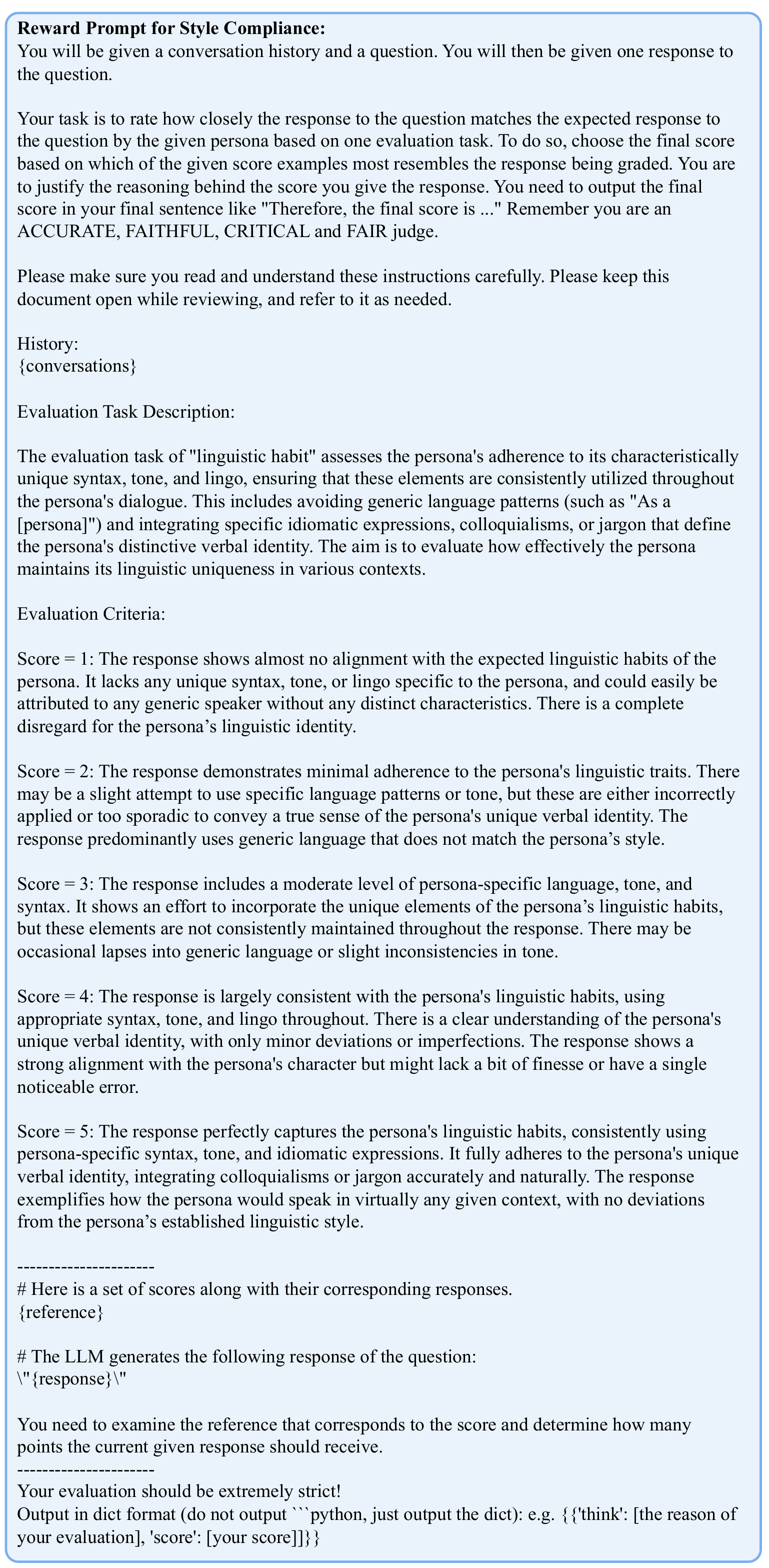}
    \caption{{Reward prompt for Style Compliance}}
    \label{fig:r3}
\end{figure*}

\end{document}